\documentclass[nohyperref]{article}

\usepackage{microtype}
\usepackage{graphicx}
\usepackage{booktabs} 

\usepackage{hyperref}

\usepackage[accepted]{icml2022}

\usepackage{amsmath}
\usepackage{amssymb}
\usepackage{amsfonts}       
\usepackage{mathtools}
\usepackage{amsthm}
\usepackage{bm}
\usepackage{pifont}
\usepackage{enumitem}
\usepackage{nicefrac}       
\usepackage{color, colortbl}

\usepackage{subcaption}
\usepackage{graphics}
\usepackage{graphicx}
\usepackage{caption}
\usepackage{wrapfig}

\usepackage{adjustbox}
\usepackage{algorithm}
\usepackage{algorithmic}
\usepackage{multirow}

\usepackage[capitalize,noabbrev]{cleveref}

\theoremstyle{plain}
\newtheorem{theorem}{Theorem}
\newtheorem{proposition}{Proposition}
\newtheorem{lemma}{Lemma}
\newtheorem{corollary}{Corollary}
\theoremstyle{definition}

\theoremstyle{remark}
\newtheorem{remark}{Remark}

\newcommand*\diff{\mathop{}\!\mathrm{d}}
\newcommand{\cmark}{\ding{51}}%
\newcommand{\xmark}{\ding{55}}%
\definecolor{Gray}{gray}{0.9}
\newcolumntype{g}{>{\columncolor{Gray}}c}
\newcolumntype{f}{>{\columncolor{Gray}}l}
\makeatletter
\def\smallunderbrace#1{\mathop{\vtop{\m@th\ialign{##\crcr
   $\hfil\displaystyle{#1}\hfil$\crcr
   \noalign{\kern3\p@\nointerlineskip}%
   \tiny\upbracefill\crcr\noalign{\kern3\p@}}}}\limits}
\makeatother
\makeatletter
\def\smalloverbrace#1{\mathop{\vtop{\m@th\ialign{##\crcr
   $\hfil\displaystyle{#1}\hfil$\crcr
   \noalign{\kern3\p@\nointerlineskip}%
   \tiny\upbracefill\crcr\noalign{\kern3\p@}}}}\limits}
\makeatother
\allowdisplaybreaks

\usepackage[textsize=tiny]{todonotes}

\icmltitlerunning{Soft Truncation: A Universal Training Technique of Score-based Diffusion Model}

\begin{document}

\twocolumn[
\icmltitle{Soft Truncation: A Universal Training Technique of \\Score-based Diffusion Model for High Precision Score Estimation}




\begin{icmlauthorlist}
	\icmlauthor{Dongjun Kim}{kaist}
	\icmlauthor{Seungjae Shin}{kaist}
	\icmlauthor{Kyungwoo Song}{uos}
	\icmlauthor{Wanmo Kang}{kaist}
	\icmlauthor{Il-Chul Moon}{kaist,summary_ai}
\end{icmlauthorlist}

\icmlaffiliation{kaist}{KAIST, South Korea}
\icmlaffiliation{uos}{University of Seoul, South Korea}
\icmlaffiliation{summary_ai}{Summary.AI}

\icmlcorrespondingauthor{Dongjun Kim}{dongjoun57@kaist.ac.kr}

\icmlkeywords{Machine Learning, ICML}

\vskip 0.3in
]



\printAffiliationsAndNotice{}  

\begin{abstract}
Recent advances in diffusion models bring state-of-the-art performance on image generation tasks. However, empirical results from previous research in diffusion models imply an inverse correlation between density estimation and sample generation performances. This paper investigates with sufficient empirical evidence that such inverse correlation happens because density estimation is significantly contributed by small diffusion time, whereas sample generation mainly depends on large diffusion time. However, training a score network well across the entire diffusion time is demanding because the loss scale is significantly imbalanced at each diffusion time. For successful training, therefore, we introduce Soft Truncation, a universally applicable training technique for diffusion models, that softens the fixed and static truncation hyperparameter into a random variable. In experiments, Soft Truncation achieves state-of-the-art performance on CIFAR-10, CelebA, CelebA-HQ $256\times 256$, and STL-10 datasets.
\end{abstract}

\section{Introduction}

Recent advances in generative models enable the creation of highly realistic images. One direction of such modeling is \textit{likelihood-free models} \citep{karras2019style} based on minimax training. The other direction is \textit{likelihood-based models}, including VAE \cite{vahdat2020nvae}, autoregressive models \citep{parmar2018image}, and flow models \citep{grcic2021densely}. Diffusion models \citep{ho2020denoising} are one of the most successful \textit{likelihood-based models}, where the reverse diffusion models the generative process. The success of diffusion models achieves state-of-the-art performance in image generation \citep{dhariwal2021diffusion}.

Previously, a model with the emphasis on Fr\'echet Inception Distance (FID), such as DDPM \cite{ho2020denoising} and ADM \cite{dhariwal2021diffusion}, trains the score network with the variance weighting; whereas a model with the emphasis on Negative Log-Likelihood (NLL), such as ScoreFlow \cite{song2021maximum} and VDM \cite{kingma2021variational}, trains the score network with the likelihood weighting. Such models, however, have the trade-off between NLL and FID: models with the emphasis on FID perform poorly on NLL, and vice versa. Instead of widely investigating the trade-off, they limit their work by separately training the score network on FID-favorable and NLL-favorable settings. This paper introduces Soft Truncation that significantly resolves the trade-off, with the NLL-favorable setting as the default training configuration. Soft Truncation reports a comparable FID against FID-favorable diffusion models while keeping NLL at the equivalent level of NLL-favorable models.

For that, we observe that the truncation hyperparameter is a significant hyperparameter that determines the overall scale of NLL and FID. This hyperparameter, $\epsilon$, is the smallest diffusion time to estimate the score function, and the score function beneath $\epsilon$ is not estimated. A model with small enough $\epsilon$ favors NLL at the sacrifice on FID, and a model with relatively large $\epsilon$ is preferable to FID but has poor NLL. Therefore, we introduce Soft Truncation, which softens the fixed and static truncation hyperparameter ($\epsilon$) into a random variable ($\tau$) that randomly selects its smallest diffusion time at every optimization step. In every mini-batch update, we sample a new smallest diffusion time, $\tau$, randomly, and the batch optimization endeavors to estimate the score function only on $[\tau,T]$, rather than $[\epsilon,T]$, by ignoring beneath $\tau$. As $\tau$ varies by mini-batch updates, the score network successfully estimates the score function on the entire range of diffusion time on $[\epsilon,T]$, which brings an improved FID. 

There are two interesting properties of Soft Truncation. First, though Soft Truncation is nothing to do with the weighting function in its algorithmic design, surprisingly, Soft Truncation turns out to be equivalent to a diffusion model with a \textit{general weight} in the expectation sense (Eq. \eqref{eq:general_weight}). The random variable of $\tau$ determines the weight function (Theorem \ref{thm:1}), and this gives a partial reason why Soft Truncation is successful in FID as much as the FID-favorable training (Table \ref{tab:ablation_weighting_function}), even though Soft Truncation only considers the truncation threshold in its implementation (Section \ref{sec:softrunc}). Second, once $\tau$ is sampled in a mini-batch optimization, Soft Truncation optimizes the log-likelihood \textit{perturbed} by $\tau$ (Lemma \ref{lemma:1}). Thus, Soft Truncation could be framed by Maximum Perturbed Likelihood Estimation (MPLE), a generalized concept of MLE that is specifically defined only in diffusion models (Section \ref{sec:MPLE}). 

\section{Preliminary}\label{sec:preliminary}

Throughout this paper, we focus on continuous-time diffusion models \cite{song2020score}. A continuous diffusion model slowly and systematically perturbs a data random variable, $\mathbf{x}_{0}$, into a noise variable, $\mathbf{x}_{T}$, as time flows. The diffusion mechanism is represented as a Stochastic Differential Equation (SDE), written by
\begin{align}\label{eq:forward_sde}
\diff\mathbf{x}_{t}=\mathbf{f}(\mathbf{x}_{t},t)\diff t+g(t)\diff\mathbf{w}_{t},
\end{align}
where $\mathbf{w}_{t}$ is a standard Wiener process. The drift ($\mathbf{f}$) and the diffusion ($g$) terms are fixed, so the data variable is diffused in a fixed manner. We denote $\{\mathbf{x}_{t}\}_{t=0}^{T}$ as the solution of the given SDE of Eq. \eqref{eq:forward_sde}, and we omit the subscript and superscript to denote $\{\mathbf{x}_{t}\}$, if no confusion is arised.

The theory of stochastic calculus indicates that there exists a corresponding reverse SDE given by
\begin{align}\label{eq:reverse_sde}
\diff\mathbf{x}_{t}=\big[\mathbf{f}(\mathbf{x}_{t},t)-g^{2}(t)\nabla\log{p_{t}(\mathbf{x}_{t})}\big]\diff\bar{t}+g(t)\diff\mathbf{\bar{w}}_{t},
\end{align}
where the solution of this reverse SDE exactly coincides to the solution of the forward SDE of Eq. \eqref{eq:forward_sde}. Here, $\diff\bar{t}$ is the backward time differential; $\diff\mathbf{\bar{w}}_{t}$ is a standard Wiener process flowing backward in time \cite{anderson1982reverse}; and $p_{t}(\mathbf{x}_{t})$ is the probability distribution of $\mathbf{x}_{t}$. Henceforth, we represent $\{\mathbf{x}_{t}\}$ as the solution of SDEs of Eqs. \eqref{eq:forward_sde} and \eqref{eq:reverse_sde}.

The diffusion model's objective is to \textit{learn} the stochastic process, $\{\mathbf{x}_{t}\}$, as a parametrized stochastic process, $\{\mathbf{x}_{t}^{\bm{\theta}}\}$. A diffusion model builds the parametrized stochastic process as a solution of a generative SDE,
\begin{align}\label{eq:generative_sde}
\diff\mathbf{x}_{t}^{\bm{\theta}}=\big[\mathbf{f}(\mathbf{x}_{t}^{\bm{\theta}},t)-g^{2}(t)\mathbf{s}_{\bm{\theta}}(\mathbf{x}_{t}^{\bm{\theta}},t)\big]\diff\bar{t}+g(t)\diff\mathbf{\bar{w}}_{t}.
\end{align}
We construct the parametrized stochastic process by solving the generative SDE of Eq. \eqref{eq:generative_sde} backward in time with a starting variable of $\mathbf{x}_{T}^{\bm{\theta}}\sim \pi$, where $\pi$ is an noise distribution. Throughout the paper, we denote $p_{t}^{\bm{\theta}}$ as the probability distribution of $\mathbf{x}_{t}^{\bm{\theta}}$.

A diffusion model learns the generative stochastic process by minimizing the score loss \cite{song2021maximum} of
\begin{align*}
\mathcal{L}(\bm{\theta};\lambda)=\frac{1}{2}\int_{0}^{T}\lambda(t) \mathbb{E}_{\mathbf{x}_{t}}\big[\Vert\mathbf{s}_{\bm{\theta}}(\mathbf{x}_{t},t)-\nabla\log{p_{t}(\mathbf{x}_{t})}\Vert_{2}^{2}\big] \diff t,
\end{align*}
where $\lambda(t)$ is a weighting function that counts the contribution of each diffusion time on the loss function. This score loss is infeasible to optimize because the data score, $\nabla\log{p_{t}(\mathbf{x}_{t})}$, is intractable in general. Fortunately, $\mathcal{L}(\bm{\theta};\lambda)$ is known to be equivalent to the (continuous) denoising NCSN loss \cite{song2020score, song2019generative},
\begin{align*}
&\mathcal{L}_{NCSN}(\bm{\theta};\lambda)\\
&=\frac{1}{2}\int_{0}^{T}\lambda(t) \mathbb{E}_{\mathbf{x}_{0},\mathbf{x}_{t}}\big[\Vert\mathbf{s}_{\bm{\theta}}(\mathbf{x}_{t},t)-\nabla\log{p_{0t}(\mathbf{x}_{t}\vert\mathbf{x}_{0})}\Vert_{2}^{2}\big] \diff t,
\end{align*}
up to a constant that is irrelevant to $\bm{\theta}$-optimization.

Two important SDEs are known to attain analytic transition probabilities, $\log{p_{0t}(\mathbf{x}_{t}\vert\mathbf{x}_{0})}$: Variance Exploding SDE (VESDE) and Variance Preserving SDE (VPSDE) \cite{song2020score}. First, VESDE assumes $\mathbf{f}(\mathbf{x}_{t},t)=0$ and $g(t)=\sigma_{min}(\frac{\sigma_{max}}{\sigma_{min}})^{t}\sqrt{2\log{\frac{\sigma_{max}}{\sigma_{min}}}}$. With such specific forms of $\mathbf{f}$ and $g$, the transition probability of VESDE turns out to follow a Gaussian distribution of $p_{0t}(\mathbf{x}_{t}\vert\mathbf{x}_{0})=\mathcal{N}(\mathbf{x}_{t};\mu_{VE}(t)\mathbf{x}_{0},\sigma_{VE}^{2}(t)\mathbf{I})$ with $\mu_{VE}(t)\equiv 1$ and $\sigma_{VE}^{2}(t)=\sigma_{min}^{2}[(\frac{\sigma_{max}}{\sigma_{min}})^{2t}-1]$. Similarly, VPSDE takes $\mathbf{f}(\mathbf{x}_{t},t)=-\frac{1}{2}\beta(t)\mathbf{x}_{t}$ and $g(t)=\sqrt{\beta(t)}$, where $\beta(t)=\beta_{min}+t(\beta_{max}-\beta_{min})$; and its transition probability falls into a Gaussian distribution of $p_{0t}(\mathbf{x}_{t}\vert\mathbf{x}_{0})=\mathcal{N}(\mathbf{x}_{t};\mu_{VP}(t)\mathbf{x}_{0},\sigma_{VP}^{2}\mathbf{I})$ with $\mu_{VP}(t)=e^{-\frac{1}{2}\int_{0}^{t}\beta(s)\diff s}$ and $\sigma_{VP}^{2}(t)=1-e^{-\int_{0}^{t}\beta(s)\diff s}$.

\begin{figure*}[t]
\centering
	\begin{subfigure}{0.32\linewidth}
	\centering
		\includegraphics[width=\linewidth]{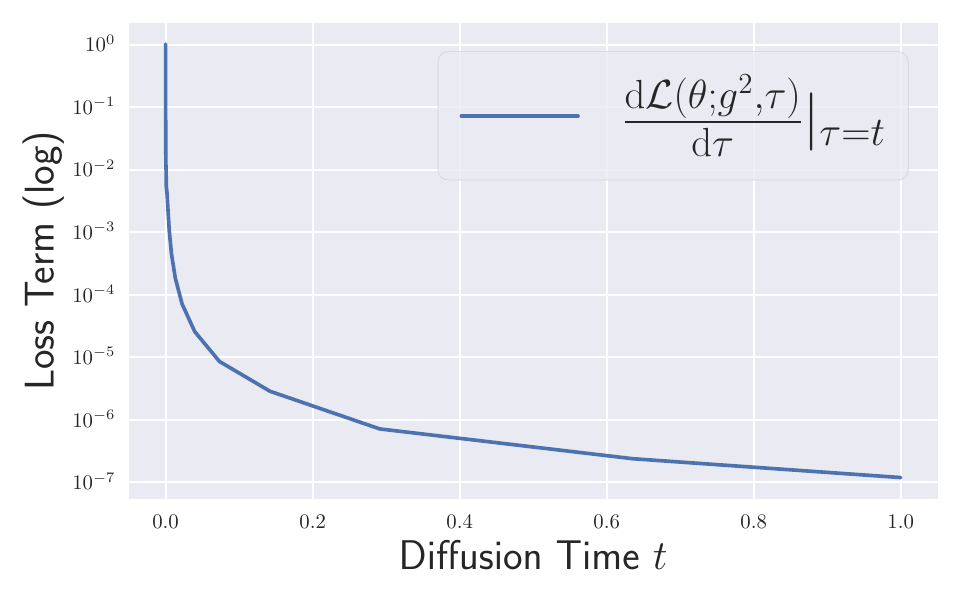}
		\subcaption{Integrand by Time}
	\end{subfigure}
	\begin{subfigure}{0.32\linewidth}
	\centering
		\includegraphics[width=\linewidth]{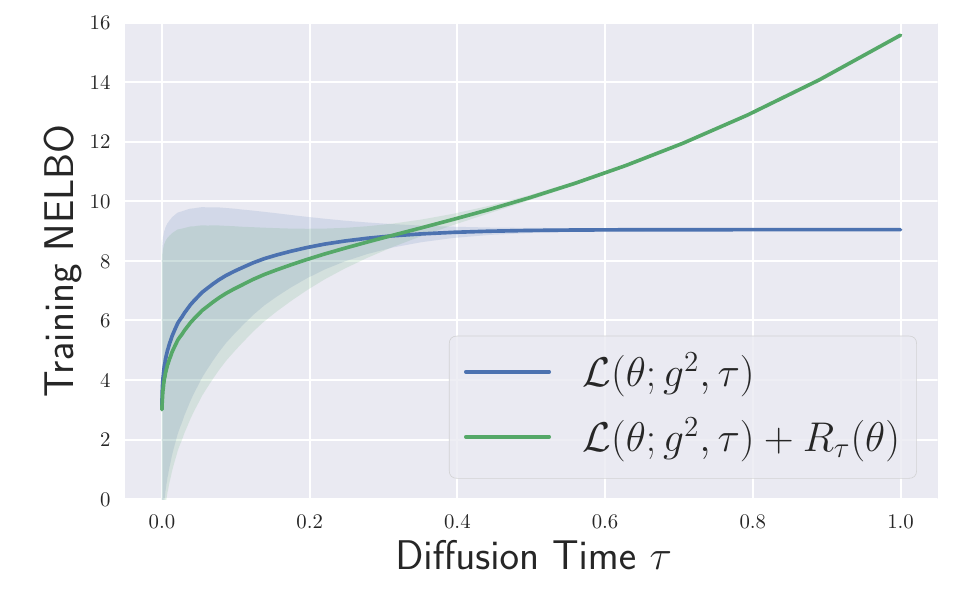}
		\subcaption{Variational Bound Truncated at $\tau$}
	\end{subfigure}
	\begin{subfigure}{0.32\linewidth}
		\includegraphics[width=\linewidth]{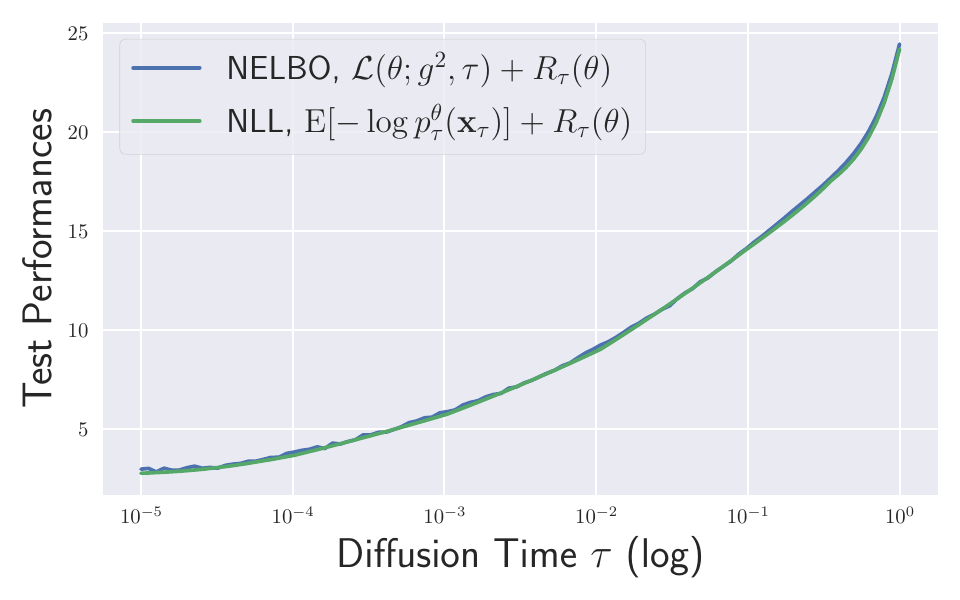}
		\subcaption{Test Performance by Log-Time}
	\end{subfigure}
	\caption{The contribution of diffusion time on the variational bound experimented on CIFAR-10 with DDPM++ (VP, NLL) \cite{song2021maximum}. (a) The integrand of the variational bound is extremely imbalanced on $[\epsilon,T]$. (b) The truncated variational bound only changes near $\tau\approx 0$. (c) The truncation hyperparameter ($\epsilon$) is a significant factor for performances.}
	\label{fig:nelbo}
\end{figure*}

Recently, \citet{kim2022maximum} categorize VESDE and VPSDE as a family of linear diffusions that has the SDE of
\begin{align}\label{eq:forward_linear_sde}
\diff\mathbf{x}_{t}=-\frac{1}{2}\beta(t)\mathbf{x}_{t}\diff t+g(t)\diff\mathbf{w}_{t},
\end{align}
where $\beta(t)$ and $g(t)$ are generic $t$-functions. Under the linear diffusions, we derive the transition probability to follow a Gaussian distribution $p_{0t}(\mathbf{x}_{t}\vert\mathbf{x}_{0})=\mathcal{N}(\mathbf{x}_{t};\mu(t)\mathbf{x}_{0},\sigma^{2}(t)\mathbf{I})$ for certain $\mu(t)$ and $\sigma(t)$ depending on $\beta(t)$ and $g(t)$, respectively (see Eq. \eqref{eq:appendix_transition_probability} of Appendix \ref{sec:transition_probability}). We emphasize that the suggested Soft Truncation is applicable for any SDE of Eq. \eqref{eq:forward_sde}, but we limit our focus to the family of linear SDEs of Eq. \eqref{eq:forward_linear_sde}, particularly VESDE and VPSDE among linear SDEs, to maintain the simplicity. With such a Gaussian transtion probability, the denoising NCSN loss with a linear SDE is equivalent to
\begin{align*}
\frac{1}{2}\int_{0}^{T}\frac{\lambda(t)}{\sigma^{2}(t)}\mathbb{E}_{\mathbf{x}_{0},\bm{\epsilon}}\big[\Vert\bm{\epsilon}_{\bm{\theta}}(\mu(t)\mathbf{x}_{0}+\sigma(t)\bm{\epsilon},t)-\bm{\epsilon}\Vert_{2}^{2}\big]\diff t,
\end{align*}
if $\bm{\epsilon}_{\bm{\theta}}(\mu(t)\mathbf{x}_{0}+\sigma(t)\bm{\epsilon},t)=-\sigma(t)\mathbf{s}_{\bm{\theta}}(\mu(t)\mathbf{x}_{0}+\sigma(t)\bm{\epsilon},t)$, where $\bm{\epsilon}\sim\mathcal{N}(0,\mathbf{I})$ is a random perturbation, and $\bm{\epsilon}_{\bm{\theta}}$ is the neural network that predicts $\bm{\epsilon}$. This is the (continuous) DDPM loss \cite{song2020score}, and the equivalence of the two losses provides a unified view of NCSN and DDPM. Hence, NCSN and DDPM are exchangeable for each other, and we take the NCSN loss as a default form of a diffusion loss throughout the paper.

The NCSN loss training is connected to the likelihood training in \citet{song2021maximum} by
\begin{align}\label{eq:variational_bound_of_song}
\mathbb{E}_{\mathbf{x}_{0}}[-\log{p_{0}^{\bm{\theta}}(\mathbf{x}_{0})}]\le\mathcal{L}_{NCSN}(\bm{\theta};g^{2}),
\end{align}
when the weighting function is the square of the diffusion term as $\lambda(t)=g^{2}(t)$, called the likelihood weighting. 

\section{Training and Evaluation of Diffusion Models in Practice}\label{sec:practice}

\subsection{The Need of Truncation}\label{sec:diverging}

In the family of linear SDEs, the gradient of the log transition probability satisfies $\nabla\log{p_{0t}(\mathbf{x}_{t}\vert\mathbf{x}_{0})}=-\frac{\mathbf{x}_{t}-\mu(t)\mathbf{x}_{0}}{\sigma^{2}(t)}=-\frac{\mathbf{z}}{\sigma(t)}$, where $\mathbf{x}_{t}$ is given to $\mu(t)\mathbf{x}_{0}+\sigma(t)\mathbf{z}$ with $\mathbf{z}\sim\mathcal{N}(0,\mathbf{I})$. The denominator of $\sigma(t)$ converges to zero as $t\rightarrow 0$, which leads $\Vert\mathbf{s}_{\bm{\theta}}(\mathbf{x}_{t},t)-\nabla\log{p_{0t}(\mathbf{x}_{t}\vert\mathbf{x}_{0})}\Vert_{2}$ to diverge as $t\rightarrow 0$, as illustrated in Figure \ref{fig:nelbo}-(a), see Appendix \ref{sec:score_fail} for details. Therefore, the Monte-Carlo estimation of the NCSN loss is under high variance, which prevents stable training of the score network. In practice, therefore, previous research truncates the diffusion time range to $[\tau,T]$, with a positive truncation hyperparameter, $\tau=\epsilon>0$.

\subsection{Variational Bound With Positive Truncation}

For the analysis for density estimation in Section \ref{sec:universal}, this section derives the variational bound of the log-likelihood when a diffusion model has a positive truncation because Inequality \eqref{eq:variational_bound_of_song} holds only with zero truncation ($\tau=0$). Lemma \ref{lemma:1} provides a generalization of Inequality \eqref{eq:variational_bound_of_song}, proved by applying the data processing inequality \cite{gerchinovitz2020fano} and the Girsanov theorem \cite{pavon1991free,vargas2021solving,song2021maximum}.
\begin{lemma}\label{lemma:1}
For any $\tau\in[0,T]$,
\begin{align}\label{eq:perturbed_nelbo}
\mathbb{E}_{\mathbf{x}_{\tau}}\big[-\log{p_{\tau}^{\bm{\theta}}(\mathbf{x}_{\tau})}\big]\le\mathcal{L}(\bm{\theta};g^{2},\tau)
\end{align}
holds, where $\mathcal{L}(\bm{\theta};g^{2},\tau)=\frac{1}{2}\int_{\tau}^{T}g^{2}(t)\mathbb{E}_{\mathbf{x}_{0},\mathbf{x}_{t}}\big[\Vert\mathbf{s}_{\bm{\theta}}(\mathbf{x}_{t},t)-\nabla\log{p_{0t}(\mathbf{x}_{t}\vert\mathbf{x}_{0})}\Vert_{2}^{2}\big]\diff t$, up to a constant, see Eq. \eqref{eq:appendix_denoising_loss}. 
\end{lemma}

Lemma \ref{lemma:1} is a generalization of Inequality \eqref{eq:variational_bound_of_song} in that Inequality \eqref{eq:perturbed_nelbo} collapses to Inequality \eqref{eq:variational_bound_of_song} under the zero truncation: $\mathcal{L}_{NCSN}(\bm{\theta};\lambda)=\mathcal{L}(\bm{\theta};\lambda,\tau=0)$. If the time range is truncated to $[\tau,T]$ for $\tau\in[0,T]$, then from the variational inference, the log-likelihood becomes 
\begin{align}\label{eq:VI}
\mathbb{E}_{\mathbf{x}_{0}}\big[-\log{p_{0}^{\bm{\theta}}(\mathbf{x}_{0})}\big]\le\mathbb{E}_{\mathbf{x}_{\tau}}\big[-\log{p_{\tau}^{\bm{\theta}}(\mathbf{x}_{\tau})}\big]+R_{\tau}(\bm{\theta})
\end{align}
where
\begin{align*}
R_{\tau}(\bm{\theta})=\mathbb{E}_{\mathbf{x}_{0}}\bigg[\int p_{0\tau}(\mathbf{x}_{\tau}\vert\mathbf{x}_{0}) \log{\frac{p_{0\tau}(\mathbf{x}_{\tau}\vert\mathbf{x}_{0})}{p_{\bm{\theta}}(\mathbf{x}_{0}\vert\mathbf{x}_{\tau})}} \diff\mathbf{x}_{\tau}\bigg],
\end{align*}
with $p_{\bm{\theta}}(\mathbf{x}_{0}\vert\mathbf{x}_{\tau})$ being the probability distribution of $\mathbf{x}_{0}$ given $\mathbf{x}_{\tau}$ and the score estimation with $\mathbf{s}_{\bm{\theta}}$ at $\tau$. For any $\tau$, we apply Lemma \ref{lemma:1} to the right-hand-side of Inequality \eqref{eq:VI} to obtain the variational bound of the log-likelihood as 
\begin{align}\label{eq:nelbo_st}
\mathbb{E}_{\mathbf{x}_{0}}\big[-\log{p_{0}^{\bm{\theta}}(\mathbf{x}_{0})}\big]\le \mathcal{L}(\bm{\theta};g^{2},\tau) +R_{\tau}(\bm{\theta}).
\end{align}

\begin{figure}[t]
\vskip -0.2in
\centering
\includegraphics[width=\linewidth]{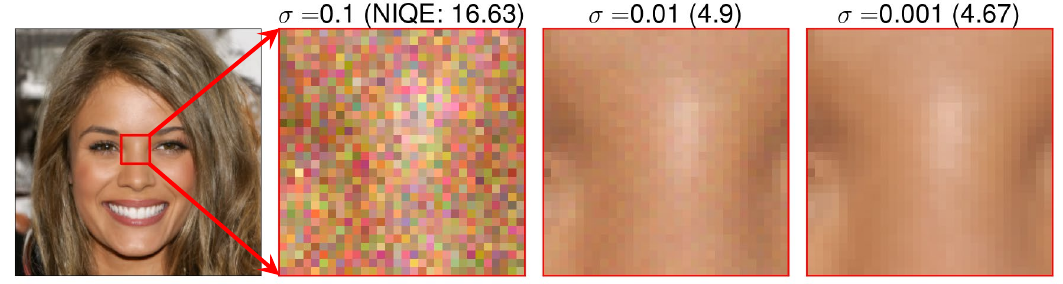}
\caption{The truncation time is key to enhance the microscopic sample quality.}
\label{fig:truncation_variance}
\end{figure}

\subsection{A Universal Phenomenon in Diffusion Training: Extremely Imbalanced Loss}\label{sec:universal}

To avoid the diverging issue introduced in Section \ref{sec:diverging}, previous works in VPSDE \cite{song2021maximum, vahdat2021score} modify the loss by truncating the integration on $[\tau,T]$ with a fixed hyperparameter $\tau=\epsilon>0$ so that the score network does not estimate the score function on $[0,\epsilon)$. Analogously, previous works in VESDE \cite{song2020score, chen2021likelihood} approximate $\sigma_{VE}^{2}(t)\approx\sigma_{min}^{2}(\frac{\sigma_{max}}{\sigma_{min}})^{2t}$ to truncate the minimum variance of the transition probability to be $\sigma_{min}^{2}$. Truncating diffusion time at $\epsilon$ in VPSDE is equivalent to truncating diffusion variance ($\sigma_{min}^{2}$) in VESDE, so these two truncations on VE/VP SDEs have the identical effect on bounding the diffusion loss. Henceforth, this paper discusses the argument of truncating diffusion time (VPSDE) and diffusion variance (VESDE) exchangeably. 

Figure \ref{fig:nelbo} illustrates the significance of truncation in the training of diffusion models. With the truncation of strictly positive $\epsilon=10^{-5}$, Figure \ref{fig:nelbo}-(a) shows that the integrand of $\mathcal{L}(\bm{\theta};g^{2},\tau)$ in the Bits-Per-Dimension (BPD) scale is still extremely imbalanced. It turns out that such extreme imbalance appears to be a universal phenomenon in training a diffusion model, and this phenomenon lasts from the beginning to the end of training.

\begin{figure}[t]
	\centering
		\includegraphics[width=\linewidth]{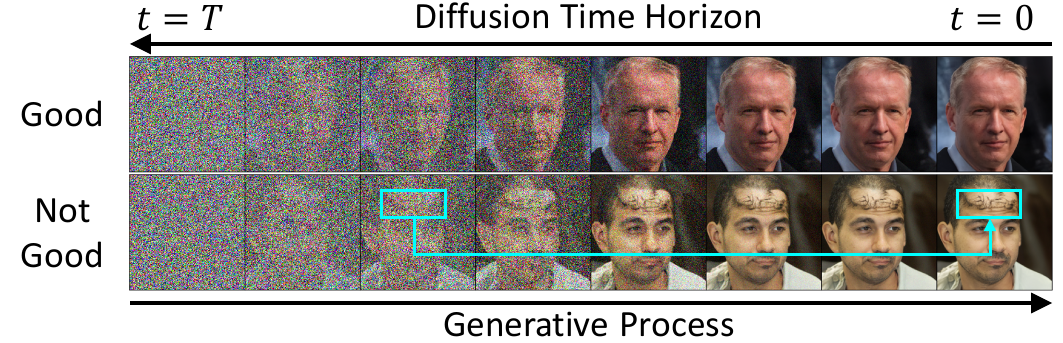}
	\caption{Illustration of the generative process trained on CelebA-HQ $256\times 256$ with NCSN++ (VE) \cite{song2020score}. The score precision on large diffusion time is key to construct the realistic overall sample quality.}
	\label{fig:large_diffusion_time}
\end{figure}

\begin{figure}[t]
\centering
\includegraphics[width=0.65\linewidth]{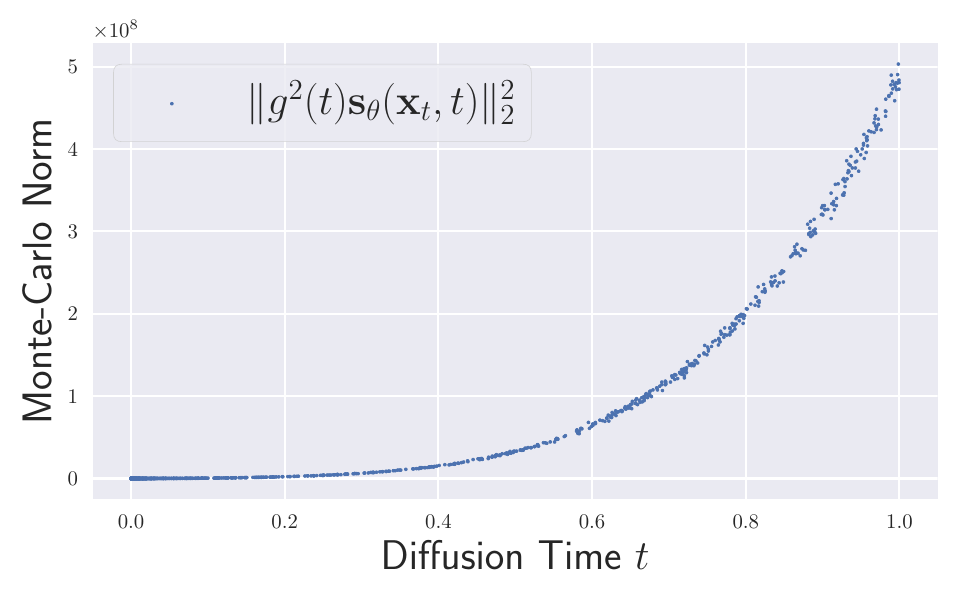}
\caption{Norm of reverse drift of generative process, trained on CIFAR-10 with DDPM++ (VP, FID) \cite{song2020score}.}
\label{fig:monte_carlo_norm}
\end{figure}

Figure \ref{fig:nelbo}-(b) with the green line presents the variational bound of the log-likelihood (right-hand-side of Inequality \eqref{eq:nelbo_st}) on the $y$-axis, and it indicates that the variational bound is sharply decreasing near the small diffusion time. Therefore, if $\epsilon$ is insufficiently small, the variational bound is not tight to the log-likelihood, and a diffusion model fails at MLE training. In addition, Figure \ref{fig:truncation_variance} indicates that insufficiently small $\epsilon$ (or $\sigma_{min}$) would also harm the microscopic sample quality. From these observations, $\epsilon$ becomes a significant hyperparameter that needs to be selected carefully.

\subsection{Effect of Truncation on Model Evaluation}

Figure \ref{fig:nelbo}-(c) reports test performances on density estimation. Figure \ref{fig:nelbo}-(c) illustrates that both Negative Evidence Lower Bound (NELBO) and NLL monotonically decrease by lowering $\epsilon$ because NELBO is largely contributed by small diffusion time at test time as well as training time. Therefore, it could be a common strategy to reduce $\epsilon$ as much as possible to reduce test NELBO/NLL. 

\begin{figure}[t]
		\centering
		\includegraphics[width=\linewidth]{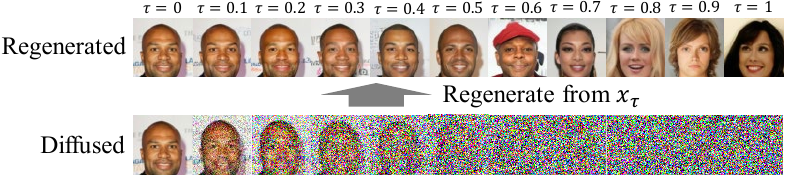}
		\caption{Regenerated samples synthesized by solving the probability flow ODE on $[\epsilon,\tau]$ backwards with the initial point of $\mathbf{x}_{\tau}=\mu(\tau)\mathbf{x}_{0}+\sigma(\tau)\mathbf{z}$ for $\mathbf{z}\sim\mathcal{N}(0,\mathbf{I})$, trained on CelebA with DDPM++ (VP, FID) \cite{song2020score}.}
		\label{fig:regenerated}
		\vskip -0.1in
\end{figure}

\begin{wraptable}{r}{0.22\textwidth}
\vskip -0.13in
\centering
	\caption{Ablation on $\sigma_{min}$.}
	\label{tab:cifar10_ablation_ncsn}
	\tiny
	\begin{tabular}{ccc}
		\toprule
		\multirow{2}{*}{$\sigma_{min}$} & \multicolumn{2}{c}{CIFAR-10}\\
		& NLL ($\downarrow$) & FID-10k ($\downarrow$) \\\midrule
		$10^{-2}$ & 4.95 & 6.95 \\
		$10^{-3}$ & 3.04 & 7.04 \\
		$10^{-4}$ & 2.99 & 8.17 \\
		$10^{-5}$ & 2.97 & 8.29 \\
		\bottomrule
	\end{tabular}
\end{wraptable}
On the contrary, there is a counter effect on FID for $\epsilon$. Table \ref{tab:cifar10_ablation_ncsn}, trained on CIFAR-10 \cite{krizhevsky2009learning} with NCSN++ \cite{song2020score}, presents that FID is worsened as we take smaller hyperparameter $\sigma_{min}$ for the training. It is the range of small diffusion time that significantly contributes to the variational bound in the blue line of Figure \ref{fig:nelbo}-(b), so the score network with a small truncation hyperparameter, $\sigma_{min}$ or $\epsilon$, remains unoptimized on large diffusion time. In the lens of Figure \ref{fig:truncation_variance}, therefore, the inconsistent result of Table \ref{tab:cifar10_ablation_ncsn} is attributed to the inaccurate score on large diffusion time.

\begin{wraptable}{r}{0.26\textwidth}
\vskip -0.1in
\centering
	\caption{FID-10k scores.}
	\label{tab:cifar10_ablation_ncsn_truncation}
	\tiny
	\begin{tabular}{llll}
		\toprule
		$\sigma_{min}$ & $10^{-3}$ & $10^{-4}$ & $10^{-5}$ \\\midrule
		$\sigma_{tr}=1$ & 6.84 & 8.04 & 8.29 \\
		\bottomrule
	\end{tabular}
\end{wraptable}
We design an experiment to validate the above argument in Table \ref{tab:cifar10_ablation_ncsn_truncation}. This experiment utilizes two types of score networks: 1) three alternative networks (As) with diverse $\sigma_{min}\in\{10^{-3},10^{-4},10^{-5}\}$ trained in Table \ref{tab:cifar10_ablation_ncsn} experiment; 2) a network (B) with $\sigma_{min}=10^{-5}$ (the last row of Table \ref{tab:cifar10_ablation_ncsn}). With these score networks, we denoise the noises by either one of the first-typed As from $\sigma_{max}$ to a common and fixed $\sigma_{tr}(=1)$, and we use B to further denoise from $\sigma_{tr}$ to $\sigma_{min}=10^{-5}$. This further denoising step with model B enables us to compare the score accuracy on large diffusion time for models with diverse truncation hyperparameters in a fair resolution setting. Table \ref{tab:cifar10_ablation_ncsn_truncation} presents that the model with $\sigma_{min}=10^{-3}$ has the best FID, implying that the training with too small truncation would harm the sample fidelity.

\begin{figure*}[t]
\centering
	\begin{subfigure}{0.32\linewidth}
	\includegraphics[width=\linewidth]{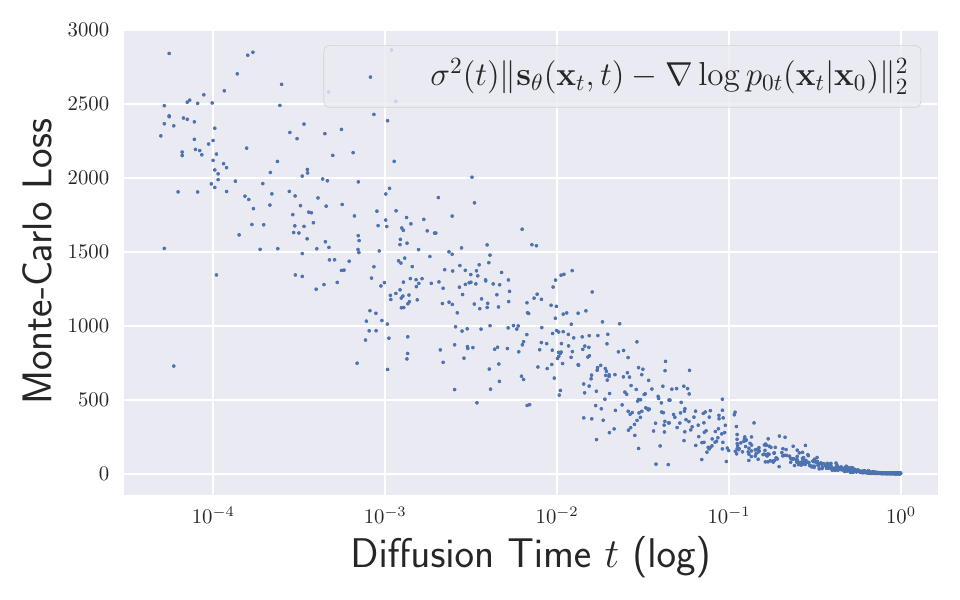}
	\subcaption{Monte-Carlo Loss}
	\end{subfigure}
	\begin{subfigure}{0.32\linewidth}
	\includegraphics[width=\linewidth]{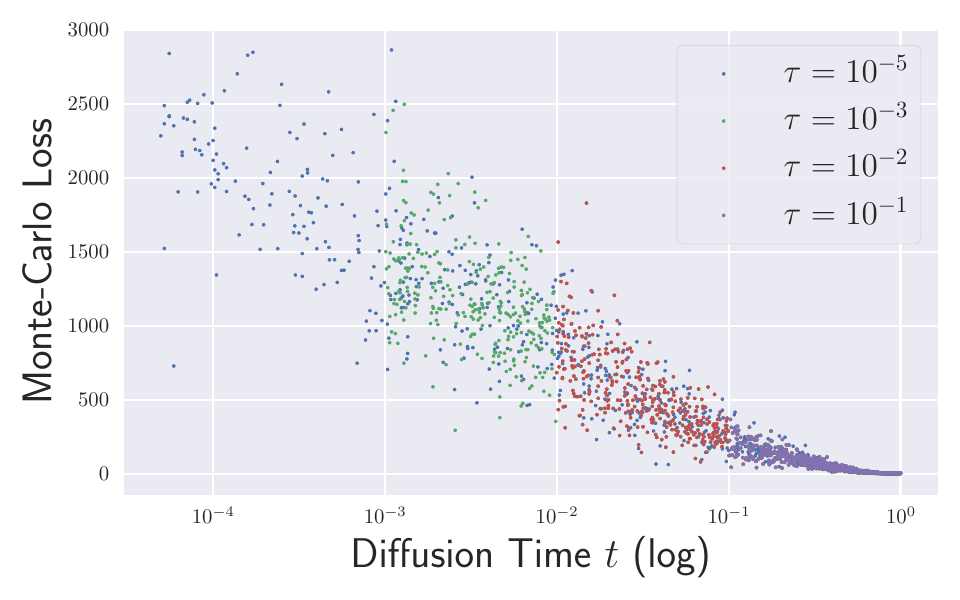}
	\subcaption{Soft Truncation}
	\end{subfigure}
	\begin{subfigure}{0.32\linewidth}
	\includegraphics[width=\linewidth]{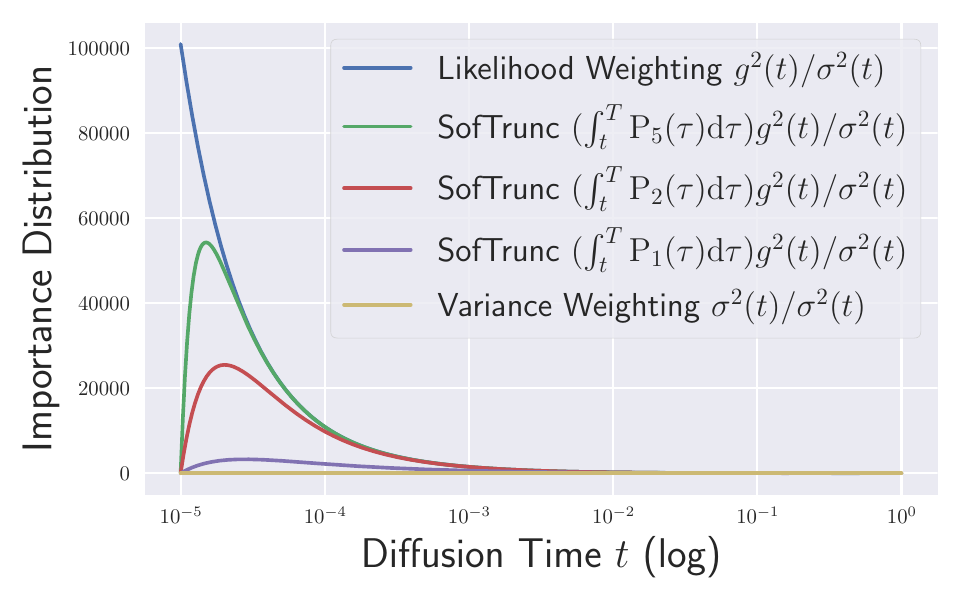}
	\subcaption{Importance Distribution}
	\end{subfigure}
	\caption{The experimental result trained on CIFAR-10 with DDPM++ (VP, NLL) \cite{song2021maximum}. (a) The Monte-Carlo loss for each diffusion time, $\sigma^{2}(t)\Vert\mathbf{s}_{\bm{\theta}}(\mathbf{x}_{t},t)-\nabla\log{p_{0t}(\mathbf{x}_{t}\vert\mathbf{x}_{0})}\Vert_{2}^{2}$. (b) The Monte-Carlo loss for each diffusion time on variaous truncation time. (c) The importance distribution for various truncation distributions.}
	\label{fig:monte_carlo_loss}
\end{figure*}

Specifically, Figure \ref{fig:monte_carlo_norm} shows the Euclidean norm of $g^{2}(t)\mathbf{s}_{\bm{\theta}}(\mathbf{x}_{t},t)$, where each dot represents for a Monte-Carlo sample from $p_{t}(\mathbf{x}_{t})$. Here, $g^{2}(t)\mathbf{s}_{\bm{\theta}}(\mathbf{x}_{t},t)$ is in the reverse drift term of the generative process, $\diff\mathbf{x}_{t}^{\bm{\theta}}=[\mathbf{f}(\mathbf{x}_{t}^{\bm{\theta}},t)-g^{2}(t)\mathbf{s}_{\bm{\theta}}(\mathbf{x}_{t}^{\bm{\theta}},t)]\diff\bar{t}+g(t)\diff\mathbf{\bar{w}}_{t}$. Figure \ref{fig:monte_carlo_norm} illustrates that it is the large diffusion time that dominates the sampling process. Therefore, a precise score network on large diffusion time is particularly important in sample generation. 

The imprecise score mainly affects the global sample context, as the denoising on small diffusion time only crafts the image in its microscopic details, illustrated in Figures \ref{fig:large_diffusion_time} and \ref{fig:regenerated}. Figure \ref{fig:large_diffusion_time} shows how the global fidelity is damaged: a man synthesized in the second row has unrealistic curly hair on his forehead, constructed on the large diffusion time. Figure \ref{fig:regenerated} deepens the importance of learning a good score estimation on large diffusion time. It shows the regenerated samples by solving the generative process time reversely, starting from $\mathbf{x}_{\tau}$ \cite{meng2021sdedit}.

\section{Soft Truncation: A Training Technique for a Diffusion Model}

As in Section \ref{sec:practice}, the choice of $\epsilon$ is crucial for training and evaluation, but it is computationally infeasible to search for the optimal $\epsilon$. Therefore, we introduce a training technique that predominantly mediates the need for $\epsilon$-search by softening the fixed truncation hyperparameter into a truncation random variable so that the truncation time varies in every optimization step. Our approach successfully trains the score network on large diffusion time without sacrificing NLL. We explain the Monte-Carlo estimation of the variational bound in Section \ref{sec:monte_carlo}, which is the common practice of previous research but explained to emphasize how simple (though effective) Soft Truncation is, and we subsequently introduce Soft Truncation in Section \ref{sec:softrunc}.

\begin{figure*}[t]
\centering
\includegraphics[width=\linewidth]{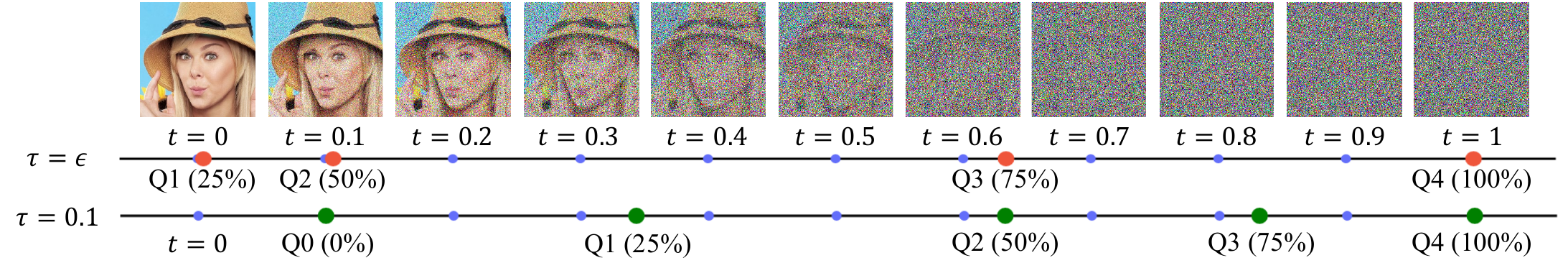}
\caption{Quartile of importance weighted Monte-Carlo time of VPSDE. Red dots represent Q1/Q2/Q3/Q4 quantiles when truncated at $\tau=\epsilon=10^{-5}$. About $25\%$ and $50\%$ of Monte-Carlo time are located in $[\epsilon,5\times 10^{-3}]$ and $[\epsilon,0.106]$, respectively. Green dots represent Q0-Q5 quantiles when truncated at $\tau=0.1$. Importance weighted Monte-Carlo time with $\tau=0.1$ is distributed much more balanced compared to the truncation at $\tau=\epsilon$.}
\label{fig:importance_sampling}
\end{figure*}

\subsection{Monte-Carlo Estimation of Truncated Variational Bound with Importance Sampling}\label{sec:monte_carlo}

In this section, we fix a truncation hyperparameter to be $\tau=\epsilon$. For every batch $\{\mathbf{x}_{0}^{(b)}\}_{b=1}^{B}$, the Monte-Carlo estimation of the variational bound in Inequality \eqref{eq:perturbed_nelbo} is $\mathcal{L}(\bm{\theta};g^{2},\epsilon)\approx\mathcal{\hat{L}}(\bm{\theta};g^{2},\epsilon)=\frac{1}{2B}\sum_{b=1}^{B}g^{2}(t^{(b)})\Vert\mathbf{s}_{\bm{\theta}}(\mathbf{x}_{t^{(b)}},t^{(b)})-\nabla\log{p_{0t^{(b)}}(\mathbf{x}_{t^{(b)}}\vert\mathbf{x}_{0})}\Vert_{2}^{2}$, up to a constant irrelevant to $\bm{\theta}$, where $\mathbf{x}_{t^{(b)}}=\mu(t^{(b)})\mathbf{x}_{0}+\sigma(t^{(b)})\bm{\epsilon}^{(b)}$ with $\{t^{(b)}\}_{b=1}^{B}$ and $\{\bm{\epsilon}^{(b)}\}_{b=1}^{B}$ be the corresponding Monte-Carlo samples from $t^{(b)}\sim [\epsilon,T]$ and $\bm{\epsilon}^{(b)}\sim\mathcal{N}(0,\mathbf{I})$, respectively. Note that this Monte-Carlo estimation is tractably computed from the analytic form of the transition probability as $\nabla\log{p_{0t^{(b)}}(\mathbf{x}_{t^{(b)}}\vert\mathbf{x}_{0})}=\frac{\bm{\epsilon}^{(b)}}{\sigma(t^{(b)})}$ under linear SDEs.

Previous works \cite{song2021maximum,huang2021variational} apply the importance sampling with the importance distribution of $p_{iw}(t)=\frac{g^{2}(t)/\sigma^{2}(t)}{Z_{\epsilon}}1_{[\epsilon,T]}(t)$, where $Z_{\epsilon}=\int_{\epsilon}^{T}\frac{g^{2}(t)}{\sigma^{2}(t)}\diff t$. It is well known \cite{goodfellow2016deep} that the Monte-Carlo variance of $\hat{\mathcal{L}}$ is minimum if the importance distribution is $p_{iw}^{*}(t)\propto g^{2}(t)L(t)$ with $L(t)=\mathbb{E}_{\mathbf{x}_{0},\mathbf{x}_{t}}[\Vert\mathbf{s}_{\bm{\theta}}(\mathbf{x}_{t},t)-\nabla\log{p_{0t}(\mathbf{x}_{t}\vert\mathbf{x}_{0})}\Vert_{2}^{2}]$, but sampling of Monte-Carlo diffusion time from $p_{iw}^{*}(t)$ at every training iteration would incur $2\times$ slower training speed, at least, because the importance sampling requires the score evaluation. Therefore, previous research approximates $L(t)$ by $\hat{L}(t)=\mathbb{E}_{\mathbf{x}_{0},\mathbf{x}_{t}}[\Vert\nabla\log{p_{0t}(\mathbf{x}_{t}\vert\mathbf{x}_{0})}\Vert_{2}^{2}]\propto 1/\sigma^{2}(t)$, and $p_{iw}(t)$ becomes the approximate importance weight. This approximation, at the expense of bias, is cheap because the closed-form of the inverse Cumulative Distribution Function (CDF) is known. Unless we train the variance directly as in \citet{kingma2021variational}, we believe $p_{iw}(t)$ is the maximally efficient sampler as long as the training speed matters. The importance weighted Monte-Carlo estimation becomes
\begin{align}
&\mathcal{L}(\bm{\theta};g^{2},\epsilon)\nonumber\\
&=\frac{Z_{\epsilon}}{2}\int_{\epsilon}^{T}p_{iw}(t)\sigma^{2}(t)\mathbb{E}\big[ \Vert \mathbf{s}_{\bm{\theta}}(\mathbf{x}_{t},t)-\nabla\log{p_{0t}(\mathbf{x}_{t}\vert\mathbf{x}_{0})} \Vert_{2}^{2} \big]\diff t\nonumber\\
&\approx\frac{Z_{\epsilon}}{2B}\sum_{b=1}^{B}\sigma^{2}(t_{iw}^{(b)})\bigg\Vert\mathbf{s}_{\bm{\theta}}\big(\mathbf{x}_{t_{iw}^{(b)}},t_{iw}^{(b)}\big)-\frac{\bm{\epsilon}^{(b)}}{\sigma(t_{iw}^{(b)})}\bigg\Vert_{2}^{2}\nonumber\\
&:=\mathcal{\hat{L}}_{iw}(\bm{\theta};g^{2},\epsilon),\label{eq:iw}
\end{align}
where $\{t_{iw}^{(b)}\}_{b=1}^{B}$ is the Monte-Carlo sample from the importance distribution, i.e., $t_{iw}^{(b)}\sim p_{iw}(t)\propto\frac{g^{2}(t)}{\sigma^{2}(t)}$. 

The importance sampling is advantageous in both NLL and FID \cite{song2021maximum} over the uniform sampling, as the importance sampling significantly reduces the estimation variance. Figure \ref{fig:monte_carlo_loss}-(a) illustrates the sample-by-sample loss, and the importance sampling significantly mitigates the loss scale by diffusion time compared to the scale in Figure \ref{fig:nelbo}-(a). However, the importance distribution satisfies $p_{iw}(t)\rightarrow\infty$ as $t\rightarrow 0$ in Figure \ref{fig:monte_carlo_loss}-(c) blue line, and most of the importance weighted Monte-Carlo time is concentrated at $t\approx \epsilon$ in Figure \ref{fig:importance_sampling}. Hence, the use of the importance sampling has a trade-off between the reduced variance (Figure \ref{fig:monte_carlo_loss}-(a)) versus the over-sampled diffusion time near $t\approx \epsilon$ (Figure \ref{fig:importance_sampling}). Regardless of whether to use the importance sampling or not, therefore, the inaccurate score estimation on large diffusion time appears sampling-strategic-independently, and solving this pre-matured score estimation becomes a nontrivial task. 

Instead of the likelihood weighting, previous works \cite{ho2020denoising, nichol2021improved, dhariwal2021diffusion} train the denoising score loss with the variance weighting, $\lambda(t)=\sigma^{2}(t)$. With this weighting, the importance distribution becomes the uniform distribution, $p_{iw}(t)=\frac{\lambda(t)}{\sigma^{2}(t)}\equiv 1$, so it significantly alleviates the trade-off of using the likelihood weighting. However, the variance weighting favors FID at the sacrifice in NLL because the loss is no longer the variational bound of the log-likelihood. In contrast, the training with the likelihood weighting is leaning towards NLL than FID, so Soft Truncation is for the \textit{balanced} NLL and FID, using the likelihood weighting.

\subsection{Soft Truncation}\label{sec:softrunc}

Soft Truncation releases the truncation hyperparameter from a static variable to a random variable with a probability distribution of $\mathbb{P}(\tau)$. In every mini-batch update, Soft Truncation optimizes the diffusion model with $\mathcal{\hat{L}}_{iw}(\bm{\theta};g^{2},\tau)$ in Eq. \eqref{eq:iw} for a sampled $\tau\sim\mathbb{P}(\tau)$. In other words, for every batch $\{\mathbf{x}_{0}^{(b)}\}_{b=1}^{B}$, Soft Truncation optimizes the Monte-Carlo loss
\begin{align*}
\mathcal{\hat{L}}_{iw}(\bm{\theta};\lambda,\tau)=\frac{Z_{\tau}}{2B}\sum_{b=1}^{B}\sigma^{2}(t_{iw}^{(b)})\bigg\Vert\mathbf{s}_{\bm{\theta}}\big(\mathbf{x}_{t_{iw}^{(b)}},t_{iw}^{(b)}\big)-\frac{\bm{\epsilon}^{(b)}}{\sigma(t_{iw}^{(b)})}\bigg\Vert_{2}^{2}
\end{align*}
with $\{t_{iw}^{(b)}\}_{b=1}^{B}$ sampled from the importance distribution of $p_{iw,\tau}(t)=\frac{g^{2}(t)/\sigma^{2}(t)}{Z_{\tau}}1_{[\tau,T]}(t)$, where $Z_{\tau}:=\int_{\tau}^{T}\frac{g^{2}(t)}{\sigma^{2}(t)}\diff t$. 

Soft Truncation resolves the oversampling issue of diffusion time near $t\approx \epsilon$, meaning that Monte-Carlo time is not concentrated on $\epsilon$ anymore. Figure \ref{fig:importance_sampling} illustrates the quantiles of importance weighted Monte-Carlo time with Soft Truncation under $\tau=\epsilon$ and $\tau=0.1$. The score network is trained more equally on diffusion time when $\tau=0.1$, and as a consequence, the loss imbalance issue in each training step is also alleviated as in Figure \ref{fig:monte_carlo_loss}-(b) with purple dots. This limited range of $[\tau,T]$ provides a chance to learn a score network more balanced on diffusion time. As $\tau$ is softened, such truncation level will vary by mini-batch updates: see the loss scales change by blue, green, red, and purple dots according to various $\tau$s in Figure \ref{fig:monte_carlo_loss}-(b). Eventually, the softened $\tau$ will provide a fair chance to learn the score network from small as well as large diffusion time.

\subsection{Soft Truncation Equals to A Diffusion Model With A General Weight}\label{sec:soft_truncation_general_weight}

In the original diffusion model, the loss estimation, $\mathcal{\hat{L}}(\bm{\theta};g^{2},\epsilon)$, is just a batch-wise approximation of a population loss, $\mathcal{L}(\bm{\theta};g^{2},\epsilon)$. However, the target population loss of Soft Truncation, $\mathcal{L}(\bm{\theta};g^{2},\tau)$, is depending on a random variable $\tau$, so the target population loss itself becomes a random variable. Therefore, we derive the \textit{expected} Soft Truncation loss to reveal the connection to the original diffusion model:
\begin{align*}
&\mathcal{L}_{ST}(\bm{\theta};g^{2},\mathbb{P}):=\mathbb{E}_{\mathbb{P}(\tau)}\big[\mathcal{L}(\bm{\theta};g^{2},\tau)\big]\\
&\quad=\frac{1}{2}\int_{\epsilon}^{T}\mathbb{P}(\tau)\int_{\tau}^{T}g^{2}(t)\mathbb{E}\big[\Vert\mathbf{s}_{\bm{\theta}}-\nabla\log{p_{0t}}\Vert_{2}^{2}\big]\diff t\diff\tau\\
&\quad=\frac{1}{2}\int_{\epsilon}^{T}g^{2}_{\mathbb{P}}(t)\mathbb{E}\big[\Vert\mathbf{s}_{\bm{\theta}}-\nabla\log{p_{0t}}\Vert_{2}^{2}\big]\diff t,
\end{align*}
up to a constant, where $g^{2}_{\mathbb{P}}(t)=\big(\int_{0}^{t}\mathbb{P}(\tau)\diff\tau\big)g^{2}(t)$, by exchanging the orders of the integrations. Therefore, we conclude that Soft Truncation reduces to a diffusion model with a general weight of $g_{\mathbb{P}}^{2}(t)$, see Appendix \ref{sec:general_weight}:
\begin{align}\label{eq:general_weight}
\mathcal{L}_{ST}(\bm{\theta};g^{2},\mathbb{P})=\mathcal{L}(\bm{\theta};g^{2}_{\mathbb{P}},\epsilon).
\end{align}

\subsection{Soft Truncation is Maximum Perturbed Likelihood Estimation}\label{sec:MPLE}

As explained in Section \ref{sec:soft_truncation_general_weight}, Soft Truncation is a diffusion model with a general weight, in the expected sense. Reversely, this section analyzes a diffusion model with a general weight in view of Soft Truncation. Suppose we have a general weight $\lambda$. Theorem \ref{thm:1} implies that this general weighted diffusion loss, $\mathcal{L}(\bm{\theta};\lambda,\epsilon)$, is the variational bound of the perturbed KL divergence expected by $\mathbb{P}_{\lambda}(\tau)$. Theorem \ref{thm:1} collapses to Lemma \ref{lemma:1} if $\lambda(t)=c g^{2}(t)$ for any $c>0$\footnote{If $\lambda(t)=cg^{2}(t)$, the probability satisfies $\mathbb{P}([a,b])=1_{[a,b]}(\epsilon)$, which is a probability distribution of one mass at $\epsilon$.}. See Appendix \ref{sec:proof} for the detailed statement and proof.
\begin{theorem}\label{thm:1}
Suppose $\frac{\lambda(t)}{g^{2}(t)}$ is a nondecreasing and nonnegative absolutely continuous function on $[\epsilon,T]$ and zero on $[0,\epsilon)$. For the probability defined by
\begin{align*}
\mathbb{P}_{\lambda}([a,b])=\bigg[\int_{\text{max}(a,\epsilon)}^{b}\Big(\frac{\lambda(s)}{g^{2}(s)}\Big)'\diff s+\frac{\lambda(\epsilon)}{g^{2}(\epsilon)}1_{[a,b]}(\epsilon)\bigg]\bigg/Z,
\end{align*}
where $Z=\frac{\lambda(T)}{g^{2}(T)}$; up to a constant, the variational bound of the general weighted diffusion loss becomes
\begin{eqnarray*}
\lefteqn{\quad\mathbb{E}_{\mathbb{P}_{\lambda}(\tau)}\big[D_{KL}(p_{\tau}\Vert p_{\tau}^{\bm{\theta}})\big]}&\\
&&\le\frac{1}{2Z}\int_{\epsilon}^{T}\lambda(t)\mathbb{E}_{\mathbf{x}_{t}}\big[\Vert\mathbf{s}_{\bm{\theta}}(\mathbf{x}_{t},t)-\nabla\log{p_{t}(\mathbf{x}_{t})}\Vert_{2}^{2}\big]\diff t\\
&&=\frac{1}{Z}\mathcal{L}(\bm{\theta};\lambda,\epsilon)=\mathbb{E}_{\mathbb{P}_{\lambda}(\tau)}\big[\mathcal{L}(\bm{\theta};g^{2},\tau)\big].
\end{eqnarray*}
\end{theorem}
The meaning of Soft Truncation becomes clearer in view of Theorem \ref{thm:1}. Instead of training the general weighted diffusion loss, $\mathcal{L}(\bm{\theta};\lambda,\epsilon)$, we optimize the \textit{truncated} variational bound, $\mathcal{L}(\bm{\theta};g^{2},\tau)$. This truncated loss upper bounds the \textit{perturbed} KL divergence, $D_{KL}(p_{\tau}\Vert p_{\tau}^{\bm{\theta}})$ by Lemma \ref{lemma:1}, and Figure \ref{fig:nelbo}-(c) indicates that the Inequality \eqref{eq:perturbed_nelbo} is nearly tight. Therefore, Soft Truncation could be interpreted as the Maximum Perturbed Likelihood Estimation (MPLE), where the perturbation level is a random variable. Soft Truncation is not MLE training because the Inequality \ref{eq:nelbo_st} is not tight as demonstrated in Figure \ref{fig:nelbo}-(b) unless $\tau$ is sufficiently small. 

Old wisdom is to minimize the loss variance if available for stable training. However, some optimization methods in the deep learning era (e.g., stochastic gradient descent) deliberately add noises to a loss function that eventually helps escape from a local optimum. Soft Truncation is categorized in such optimization methods that \textit{inflate} the loss variance by intentionally imposing auxiliary randomness on loss estimation. This randomness is represented by the outmost expectation of $\mathbb{E}_{\mathbb{P}_{\lambda}(\tau)}$, which controls the diffusion time range batch-wisely. Additionally, the loss with a sampled $\tau$ is the proxy of the perturbed KL divergence by $\tau$, so the auxiliary randomness on loss estimation is theoretically tamed, meaning that it is not a random perturbation.

\subsection{Choice of Truncation Probability Distribution}

We parametrize the probability distribution of $\tau$ by 
\begin{align}\label{eq:prior_example}
\mathbb{P}_{k}(\tau)=\frac{1/\tau^{k}}{Z_{k}}1_{[\epsilon,T]}(\tau)\propto\frac{1}{\tau^{k}},
\end{align}
where $Z_{k}=\int_{\epsilon}^{T}\frac{1}{\tau^{k}}\diff \tau$ with sufficiently small enough truncation hyperparameter. Note that it is still beneficial to remain $\epsilon$ strictly positive because a batch update with $\tau\approx 0<\epsilon$ would drift the score network away from the optimal point. Figure \ref{fig:monte_carlo_loss}-(c) illustrates the importance distribution of $\lambda_{\mathbb{P}_{k}}$ for varying $k$. From the definition of Eq. \eqref{eq:prior_example}, $\mathbb{P}_{k}(\tau)\rightarrow\delta_{\epsilon}(\tau)$ as $k\rightarrow \infty$, and this limiting delta distribution corresponds to the original diffusion model with the likelihood weighting. Figure \ref{fig:monte_carlo_loss}-(c) shows that the importance distribution of $\mathbb{P}_{k}$ with finite $k$ interpolates the likelihood weighting and the variance weighting. 

With the current simple form, we experimentally find that the sweet spot is $k\approx 1.0$ in VPSDE and $k=2.0$ in VESDE with the emphasis on the sample quality. For VPSDE, the importance distribution in Figure \ref{fig:monte_carlo_loss}-(c) is nearly equal to that of the variance weighting if $k\approx 1.0$, so Soft Truncation with $k\approx 1.0$ improves the sample fidelity, while maintaining low NLL. On the other hand, if $k$ is too small, no $\tau$ will be sampled near $\epsilon$, so it hurts both sample generation and density estimation. We leave further study on searching for the optimal distribution of $\tau$ as future work. 

\begin{table}[t]
\begin{minipage}[c]{0.5\textwidth}
\centering
	\begin{subfigure}{0.7\linewidth}
		\includegraphics[width=\linewidth]{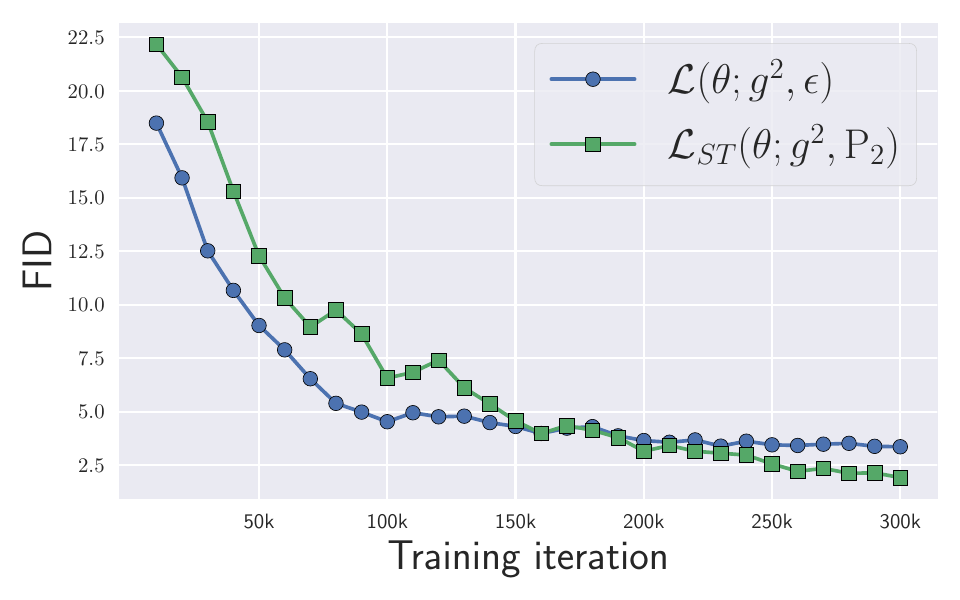}
	\end{subfigure}
	\vskip -0.1in
	\captionof{figure}{Soft Truncation improves FID on CelebA trained with UNCSN++ (RVE).}
	\label{fig:st_training}
	\end{minipage}

\begin{minipage}[c]{0.5\textwidth}
\centering
\vskip 0.1in
	\caption{Ablation study of Soft Truncation for various weightings on CIFAR-10 and ImageNet32 with DDPM++ (VP).}
	\label{tab:ablation_weighting_function}
	\vskip -0.05in
	\tiny
	\begin{tabular}{llcccc}
		\toprule
		& \multirow{2}{*}{Loss} & \multirow{2}{*}{\shortstack{Soft\\Truncation}} & \multirow{2}{*}{NLL} & \multirow{2}{*}{NELBO} & FID \\
		&&&&& ODE \\\midrule
		\multirow{4}{*}[-2pt]{CIFAR-10} & $\mathcal{L}(\bm{\theta};g^{2},\epsilon)$ & \xmark & 3.03 & 3.13 & 6.70 \\
		& $\mathcal{L}(\bm{\theta};\sigma^{2},\epsilon)$ & \xmark & 3.21 & 3.34 & 3.90 \\
		& $\mathcal{L}(\bm{\theta};g_{\mathbb{P}_{1}}^{2},\epsilon)$ & \xmark & 3.06 & 3.18 & 6.11 \\
		& $\mathcal{L}_{ST}(\bm{\theta};g^{2},\mathbb{P}_{1})$ & \cmark & \textbf{3.01} & \textbf{3.08} & 3.96 \\
		& $\mathcal{L}_{ST}(\bm{\theta};g^{2},\mathbb{P}_{0.9})$ & \cmark & 3.03 & 3.13 & \textbf{3.45} \\\midrule
		\multirow{4}{*}[-2pt]{ImageNet32} & $\mathcal{L}(\bm{\theta};g^{2},\epsilon)$ & \xmark & 3.92 & 3.94 & 12.68 \\
		& $\mathcal{L}(\bm{\theta};\sigma^{2},\epsilon)$ & \xmark & 3.95 & 4.00 & 9.22 \\
		& $\mathcal{L}(\bm{\theta};g_{\mathbb{P}_{1}}^{2},\epsilon)$ & \xmark & 3.93 & 3.97 & 11.89 \\
		& $\mathcal{L}_{ST}(\bm{\theta};g^{2},\mathbb{P}_{0.9})$ & \cmark & \textbf{3.90} & \textbf{3.91} & \textbf{8.42} \\
		\bottomrule
	\end{tabular}
\end{minipage}

\begin{minipage}[c]{0.5\textwidth}
\vskip 0.1in
\centering
	\caption{Ablation study of Soft Truncation for various model architectures and diffusion SDEs on CelebA.}
	\label{tab:ablation_architecture_sde}
	\vskip -0.05in
	\tiny
	\begin{tabular}{lllcccc}
		\toprule
		\multirow{2}{*}{SDE} & \multirow{2}{*}{Model} & \multirow{2}{*}{Loss} & \multirow{2}{*}{NLL} & \multirow{2}{*}{NELBO} & \multicolumn{2}{c}{FID} \\
		&&&&& PC & ODE \\\midrule
		\multirow{2}{*}{VE} & \multirow{2}{*}{NCSN++} & $\mathcal{L}(\bm{\theta};\sigma^{2},\epsilon)$ & 3.41 & 3.42 & 3.95 & -\\
		& & $\mathcal{L}_{ST}(\bm{\theta};\sigma^{2},\mathbb{P}_{2})$ & 3.44 & 3.44 & 2.68 & -\\\midrule
		\multirow{2}{*}{RVE} & \multirow{2}{*}{UNCSN++} & $\mathcal{L}(\bm{\theta};g^{2},\epsilon)$ & 2.01 & \textbf{2.01} & 3.36 & -\\
		& & $\mathcal{L}_{ST}(\bm{\theta};g^{2},\mathbb{P}_{2})$ & \textbf{1.97} & 2.02 & \textbf{1.92} & -\\\midrule
		\multirow{8}{*}[-9pt]{VP} & \multirow{2}{*}{DDPM++} & $\mathcal{L}(\bm{\theta};\sigma^{2},\epsilon)$  & 2.14 & 2.21 & 3.03 & 2.32 \\
		& & $\mathcal{L}_{ST}(\bm{\theta};\sigma^{2},\mathbb{P}_{1})$ & 2.17 & 2.29 & 2.88 & \textbf{1.90}\\\cmidrule(lr){2-3}
		& \multirow{2}{*}{UDDPM++} & $\mathcal{L}(\bm{\theta};\sigma^{2},\epsilon)$ & 2.11 & 2.20 & 3.23 & 4.72\\
		& & $\mathcal{L}_{ST}(\bm{\theta};\sigma^{2},\mathbb{P}_{1})$ & 2.16 & 2.28 & 2.22 & 1.94\\\cmidrule(lr){2-3}
		& \multirow{2}{*}{DDPM++} & $\mathcal{L}(\bm{\theta};g^{2},\epsilon)$ & 2.00 & 2.09 & 5.31 & 3.95\\
		& & $\mathcal{L}_{ST}(\bm{\theta};g^{2},\mathbb{P}_{1})$ & 2.00 & 2.11 & 4.50 & 2.90\\\cmidrule(lr){2-3}
		& \multirow{2}{*}{UDDPM++} & $\mathcal{L}(\bm{\theta};g^{2},\epsilon)$ & 1.98 & 2.12 & 4.65 & 3.98\\
		& & $\mathcal{L}_{ST}(\bm{\theta};g^{2},\mathbb{P}_{1})$ & 2.00 & 2.10 & 4.45 & 2.97\\
		\bottomrule
	\end{tabular}
	\end{minipage}
		\vskip -0.1in
\end{table}

\begin{table}[t]
\begin{minipage}[c]{0.5\textwidth}
\centering
	\caption{Ablation study of Soft Truncation for various $\epsilon$ on CIFAR-10 with DDPM++ (VP).}
	\label{tab:ablation_epsilon}
	\vskip -0.05in
	\tiny
	\begin{tabular}{lcccc}
		\toprule
		Loss & $\epsilon$ & NLL & NELBO & FID (ODE) \\\midrule
		\multirow{4}{*}{$\mathcal{L}(\bm{\theta};g^{2},\epsilon)$} & $10^{-2}$ & 4.64 & 4.69 & 38.82 \\
		& $10^{-3}$ & 3.51 & 3.52 & 6.21 \\
		& $10^{-4}$ & 3.05 & 3.08 & 6.33 \\
		& $10^{-5}$ & 3.03 & 3.13 & 6.70 \\\midrule
		\multirow{4}{*}{$\mathcal{L}_{ST}(\bm{\theta};g^{2},\mathbb{P}_{1})$} & $10^{-2}$ & 4.65 & 4.69 & 39.83 \\
		& $10^{-3}$ & 3.51 & 3.52 & 5.14 \\
		& $10^{-4}$ & 3.05 & 3.08 & 4.16 \\
		& $10^{-5}$ & \textbf{3.01} & \textbf{3.08} & \textbf{3.96} \\
		\bottomrule
	\end{tabular}
	\end{minipage}

\begin{minipage}[c]{0.5\textwidth}
\vskip 0.1in
\centering
\caption{Ablation study of Soft Truncation for various $\mathbb{P}_{k}$ on CIFAR-10 trained with DDPM++ (VP).}
\label{tab:ablation_prior}
	\vskip -0.05in
\tiny
\begin{tabular}{lccc}
	\toprule
	Loss & NLL & NELBO & FID (ODE) \\\midrule
	$\mathcal{L}_{ST}(\bm{\theta};g^{2},\mathbb{P}_{0})$ & 3.24 & 3.39 & 6.27 \\
	$\mathcal{L}_{ST}(\bm{\theta};g^{2},\mathbb{P}_{0.8})$ & 3.03 & \textbf{3.05} & 3.61 \\
	$\mathcal{L}_{ST}(\bm{\theta};g^{2},\mathbb{P}_{0.9})$ & 3.03 & 3.13 & \textbf{3.45} \\
	$\mathcal{L}_{ST}(\bm{\theta};g^{2},\mathbb{P}_{1})$ & \textbf{3.01} & 3.08 & 3.96 \\
	$\mathcal{L}_{ST}(\bm{\theta};g^{2},\mathbb{P}_{1.1})$ & 3.02 & 3.09 & 3.98 \\
	$\mathcal{L}_{ST}(\bm{\theta};g^{2},\mathbb{P}_{1.2})$ & 3.03 & 3.09 & 3.98 \\
	$\mathcal{L}_{ST}(\bm{\theta};g^{2},\mathbb{P}_{2})$ & \textbf{3.01} & 3.10 & 6.31 \\
	$\mathcal{L}_{ST}(\bm{\theta};g^{2},\mathbb{P}_{3})$ & 3.02 & 3.09 & 6.54 \\\midrule
	$\mathcal{L}_{ST}(\bm{\theta};g^{2},\mathbb{P}_{\infty})$ & \multirow{2}{*}{\textbf{3.01}} & \multirow{2}{*}{3.09} & \multirow{2}{*}{6.70} \\
	$=\mathcal{L}(\bm{\theta};g^{2},\epsilon)$ &&&\\
	\bottomrule
\end{tabular}
\end{minipage}

\begin{minipage}[c]{0.5\textwidth}
\vskip 0.1in
\centering
\caption{Ablation study of Soft Truncation for CIFAR-10 trained with DDPM++ when a diffusion is combined with a normalizing flow in INDM \cite{kim2022maximum}.}
\label{tab:ablation_indm}
	\vskip -0.05in
\tiny
\begin{tabular}{lccc}
	\toprule
	Loss & NLL & NELBO & FID (ODE) \\\midrule
	INDM (VP, NLL) & \textbf{2.98} & \textbf{2.98} & 6.01 \\
	INDM (VP, FID) & 3.17 & 3.23 & \textbf{3.61} \\
	INDM (VP, NLL) + ST & 3.01 & 3.02 & 3.88 \\
	\bottomrule
\end{tabular}
\end{minipage}
\end{table}

\begin{table*}
\centering
	\caption{Performance comparisons on benchmark datasets. The boldfaced numbers present the best performance, and the underlined numbers present the second-best performance. We report NLL of DDPM++ on CIFAR-10, ImageNet32, and CelebA with the variational dequantization \cite{song2021maximum} to compare with the baselines in a fair setting.}
	\label{tab:performances}
	\begin{adjustbox}{max width=\textwidth}
		\begin{tabular}{lccccccccccc}
			\toprule
			\multirow{3}{*}{Model} & \multicolumn{3}{c}{CIFAR10} & \multicolumn{3}{c}{ImageNet32} & \multicolumn{2}{c}{CelebA} & CelebA-HQ & \multicolumn{2}{c}{STL-10} \\
			& \multicolumn{3}{c}{$32\times 32$} & \multicolumn{3}{c}{$32\times 32$} & \multicolumn{2}{c}{$64\times 64$} & $256\times 256$ & \multicolumn{2}{c}{$48\times 48$} \\
			& NLL ($\downarrow$) & FID ($\downarrow$) & IS ($\uparrow$) & NLL & FID & IS & NLL & FID & FID & FID & IS \\\midrule
			\multicolumn{12}{l}{\textbf{Likelihood-free Models}}\\
			StyleGAN2-ADA+Tuning \citep{karras2020training} & - & 2.92 & \underline{10.02} & - & - & - & - & - & - & - & - \\
			Styleformer \citep{park2021styleformer} & - & 2.82 & 9.94 & - & - & - & - & 3.66 & - & \underline{15.17} & \underline{11.01} \\
			\multicolumn{12}{l}{\textbf{Likelihood-based Models}}\\
			ARDM-Upscale 4 \citep{hoogeboom2021autoregressive} & \textbf{2.64} & - & - & - & - & - & - & - & - & - & - \\
			VDM \citep{kingma2021variational} & \underline{2.65} & 7.41 & - & \underline{3.72} & - & - & - & - & - & - & - \\
			LSGM (FID) \citep{vahdat2021score} & 3.43 & \textbf{2.10} & - & - & - & - & - & - & - & - & - \\
			NCSN++ cont. (deep, VE) \citep{song2020score} & 3.45 & \underline{2.20} & 9.89 & - & - & - & 2.39 & 3.95 & \underline{7.23} & - & - \\
			DDPM++ cont. (deep, sub-VP) \citep{song2020score} & 2.99 & 2.41 & 9.57 & - & - & - & - & - & - & - & - \\
			DenseFlow-74-10 \citep{grcic2021densely} & 2.98 & 34.90 & - & \textbf{3.63} & - & - & 1.99 & - & - & - & - \\
			ScoreFlow (VP, FID) \citep{song2021maximum} & 3.04 & 3.98 & - & 3.84 & \textbf{8.34} & - & - & - & - & - & - \\
			Efficient-VDVAE \citep{hazami2022efficient} & 2.87 & - & - & - & - & - & \textbf{1.83} & - & - & - & - \\
			PNDM \citep{liu2022pseudo} & - & 3.26 & - & - & - & - & - & 2.71 & - & - & - \\
			ScoreFlow (deep, sub-VP, NLL) \citep{song2021maximum} & 2.81 & 5.40 & - & 3.76 & 10.18 & - & - & - & - & - & - \\
			Improved DDPM ($L_{simple}$) \citep{nichol2021improved} & 3.37 & 2.90 & - & - & - & - & - & - & - & - & - \\\midrule
			UNCSN++ (RVE) + ST & 3.04 & 2.33 & \textbf{10.11} & - & - & - & 1.97 & \underline{1.92} & \textbf{7.16} & \textbf{7.71} & \textbf{13.43} \\
			DDPM++ (VP, FID) + ST & 2.91 & 2.47 & 9.78 & - & - & - & 2.10 & \textbf{1.90} & - & - & - \\
			DDPM++ (VP, NLL) + ST & 2.88 & 3.45 & 9.19 & 3.85 & \underline{8.42} & \textbf{11.82} & \underline{1.96} & 2.90 & - & - & - \\
			\bottomrule
		\end{tabular}
	\end{adjustbox}
\end{table*}

\section{Experiments}

This section empirically studies our suggestions on benchmark datasets, including CIFAR-10 \citep{krizhevsky2009learning}, ImageNet $32\times 32$ \cite{van2016pixel}, STL-10 \citep{coates2011analysis}\footnote{We downsize the dataset from $96\times 96$ to $48\times 48$ following \citet{jiang2021transgan, park2021styleformer}.} CelebA \citep{liu2015deep} $64\times 64$ and CelebA-HQ \citep{karras2018progressive} $256\times 256$.

Soft Truncation is a universal training technique indepedent to model architectures and diffusion strategies. In the experiments, we test Soft Truncation on various architectures, including vanilla NCSN++, DDPM++, Unbounded NCSN++ (UNCSN++), and Unbounded DDPM++ (UDDPM++). Also, Soft Truncation is applied to various diffusion SDEs, such as VESDE, VPSDE, and Reverse VESDE (RVESDE). Although we use continuous SDEs for the diffusion strategies, Soft Truncation with the discrete model, such as DDPM \cite{ho2020denoising}, is a straightforward application of continuous models. Appendix \ref{sec:implementation_details} enumerates the specifications of score architectures and SDEs.

From Figure \ref{fig:nelbo}-(c), a sweet spot of the hard threshold is $\epsilon=10^{-5}$, in which NLL/NELBO are no longer improved under this threshold. As the diffusion model has no information on $[0,\epsilon)$, we comply \citet{kim2022maximum} to use Inequality \eqref{eq:VI} for NLL computation and Inequality \eqref{eq:nelbo_st} for NELBO computation. Following \citet{kim2022maximum}, we compute $\log{p_{\epsilon}^{\bm{\theta}}(\mathbf{x}_{\epsilon})}$, rather than $\log{p_{\epsilon}^{\bm{\theta}}(\mathbf{x}_{0})}$. It is the common practice of continuous diffusion models \cite{song2020score, song2021maximum, dockhorn2021score} to report their performances with $\log{p_{\epsilon}^{\bm{\theta}}(\mathbf{x}_{0})}$, but \citet{kim2022maximum} show that $\log{p_{\epsilon}^{\bm{\theta}}(\mathbf{x}_{\epsilon})}$ differs to $\log{p_{\epsilon}^{\bm{\theta}}(\mathbf{x}_{0})}$ by 0.05 in BPD scale when $\epsilon=10^{-5}$, which is quite significant. We use the uniform dequantization \cite{theis2016note} as default, otherwise noted. For sample generation, we use either of Predictor-Corrector (PC) sampler or Ordinary Differential Equation (ODE) sampler \cite{song2020score}. We denote $\mathcal{L}(\bm{\theta};\lambda,\epsilon)$ as the vanilla training with $\lambda$-weighting, and $\mathcal{L}_{ST}(\bm{\theta};g^{2},\mathbb{P})$ as the training by Soft Truncation with the truncation probability of $\mathbb{P}$. We additionally denote $\mathcal{L}_{ST}(\bm{\theta};\sigma^{2},\mathbb{P})$ for updating the network by the variance weighted loss per batch-wise update. We release our code at \url{https://github.com/Kim-Dongjun/Soft-Truncation}.

\textbf{FID by Iteration} Figure \ref{fig:st_training} illustrates the FID score \cite{heusel2017gans} in $y$-axis by training steps in $x$-axis. Figure \ref{fig:st_training} shows that Soft Truncation beats the vanilla training after 150k of training iterations.

\textbf{Ablation Studies} Tables \ref{tab:ablation_weighting_function}, \ref{tab:ablation_architecture_sde}, \ref{tab:ablation_epsilon}, and \ref{tab:ablation_prior} show ablation studies on various weighting functions, model architectures, SDEs, $\epsilon$s, and probability distributions of $\tau$, respectively. See Appendix \ref{sec:full_tables}. Table \ref{tab:ablation_weighting_function} shows that Soft Truncation beats or equals to the vanilla training in all performances. We highlight that Soft Truncation with $\mathbb{P}_{0.9}$ outperforms the FID-favorable model with the variance weighting with respect to FID on both CIFAR-10 and ImageNet32.

Not only comparing with the pre-existing weighting functions, such as $\lambda=g^{2}$ or $\lambda=\sigma^{2}$, Table \ref{tab:ablation_weighting_function} additionally reports the experimental result of a general weighting function of $\lambda=g_{\mathbb{P}_{1}}^{2}$. From Eq. \eqref{eq:general_weight}, Soft Truncation with $\mathbb{P}_{1}$ and the vanilla training with $\lambda=g_{\mathbb{P}_{1}}^{2}$ coincide in their loss functions in average, i.e., $\mathcal{L}(\bm{\theta};g_{\mathbb{P}_{1}}^{2},\epsilon)=\mathcal{L}_{ST}(\bm{\theta};g^{2},\mathbb{P}_{1})$. Thus, when comparing the paired experiments, Soft Truncation could be considered as an alternative way of estimating the same loss, and Table \ref{tab:ablation_weighting_function} implies that Soft Truncation gives better optimization than the vanilla method. This strongly implies that Soft Truncation could be a default training method for a general weighted denoising diffusion loss.

Table \ref{tab:ablation_architecture_sde} provides two implications. First, Soft Truncation particularly boosts FID while maintaining density estimation performances under the variation of score networks and diffusion strategies. Second, Table \ref{tab:ablation_architecture_sde} shows that Soft Truncation is effective on CelebA even when we apply Soft Truncation on the variance weighting, i.e., $\mathcal{L}_{ST}(\bm{\theta};\sigma^{2},\mathbb{P})$, but we find that this does not hold on CIFAR-10 and ImageNet32. We leave it as a future work on this extent.

Table \ref{tab:ablation_epsilon} shows a contrastive trend of the vanilla training and Soft Truncation. The inverse correlation appears between NLL and FID in the vanilla training, but Soft Truncation monotonically reduces both NLL and FID by $\epsilon$. This implies that Soft Truncation significantly reduces the effort of the $\epsilon$ search. Table \ref{tab:ablation_prior} studies the effect of the probability distribution of $\tau$ in VPSDE. It shows that Soft Truncation significantly improves FID upon the experiment of $\mathcal{L}(\bm{\theta};g^{2},\epsilon)$ on the range of $0.8\le k\le 1.2$. Finally, Table \ref{tab:ablation_indm} shows that Soft Truncation also works with a nonlinear forward SDE \cite{kim2022maximum}, so the scope of Soft Truncation is not limited to a family of linear SDEs.

\textbf{Quantitative Comparison to SOTA} Table \ref{tab:performances} compares Soft Truncation (ST) against the current best generative models. It shows that Soft Truncation achieves the state-of-the-art sample generation performances on CIFAR-10, CelebA, CelebA-HQ, and STL-10, while keeping NLL intact. In particular, we have experimented thoroughly on the CelebA dataset, and we find that Soft Truncation largely exceeds the previous best FID scores by far. In FID, Soft Truncation with DDPM++ performs 1.90, which exceeds the previous best FID of 2.92 by DDGM. Also, Soft Truncation significantly improves FID on STL-10.

\section{Conclusion}

This paper proposes a generally applicable training method for diffusion models. The suggested training method, Soft Truncation, is motivated from the observation that the density estimation is mostly counted on small diffusion time, while the sample generation is mostly constructed on large diffusion time. However, small diffusion time dominates the Monte-Carlo estimation of the loss function, so this imbalance contribution prevents accurate score learning on large diffusion time. Soft Truncation softens the truncation level at each mini-batch update, and this simple modification is connected to the general weighted diffusion loss and the concept of Maximum Perturbed Likelihood Estimation.

\section*{Acknowledgements}

This research was supported by AI Technology Development for Commonsense Extraction, Reasoning, and Inference from Heterogeneous Data(IITP) funded by the Ministry of Science and ICT(2022-0-00077). We thank Jaeyoung Byeon and Daehan Park for their fruitful mathematical advice, and Byeonghu Na for his support of the experiments.

\nocite{langley00}

\bibliography{references}
\bibliographystyle{icml2022}

\newpage
\appendix
\onecolumn

\section{Derivation}

\subsection{Transition Probability for Linear SDEs}\label{sec:transition_probability}

\citet{kim2022maximum} has classified linear SDEs as
\begin{align}\label{eq:appendix_linear_SDE}
\diff\mathbf{x}_{t}=-\frac{1}{2}\beta(t)\mathbf{x}_{t}\diff t+g(t)\diff\mathbf{w}_{t},
\end{align}
where $\beta:\mathbb{R}\rightarrow\mathbb{R}_{\ge 0}$ and $g:\mathbb{R}\rightarrow\mathbb{R}_{\ge 0}$ are real-valued functions. VESDE has $\beta(t)\equiv 0$ and $g(t)=\sqrt{\diff\sigma^{2}(t)/\diff t}=\sigma_{min}(\frac{\sigma_{max}}{\sigma_{min}})^{t}\sqrt{2\log{\frac{\sigma_{max}}{\sigma_{min}}}}$, where $\sigma_{min}$ and $\sigma_{max}$ are the minimum/maximum perturbation variances, respectively. It has the transition probability of
\begin{align*}
p_{0t}(\mathbf{x}_{t}\vert\mathbf{x}_{0})=\mathcal{N}(\mathbf{x}_{t};\mu_{VE}(t)\mathbf{x}_{0},\sigma_{VE}^{2}(t)\mathbf{I}),
\end{align*}
where $\mu_{VE}(t)\equiv 1$ and $\sigma_{VE}^{2}(t):=\sigma_{min}^{2}[(\frac{\sigma_{max}}{\sigma_{min}})^{2t}-1]$. VPSDE has $\beta(t)=\beta_{min}+(\beta_{max}-\beta_{min})t$ and $g(t)=\sqrt{\beta(t)}$ with the transition probability of
\begin{align*}
p_{0t}(\mathbf{x}_{t}\vert\mathbf{x}_{0})=\mathcal{N}(\mathbf{x}_{t};\mu_{VP}(t)\mathbf{x}_{0},\sigma_{VP}^{2}(t)\mathbf{I}),
\end{align*}
where $\mu_{VP}(t)=e^{-\frac{1}{2}\int_{0}^{t}\beta(s)\diff s}$ and $\sigma^{2}(t)=1-e^{-\int_{0}^{t}\beta(s)\diff s}$.

Analogous to VE/VP SDEs, the transition probability of the generic linear SDE of Eq. \eqref{eq:appendix_linear_SDE} is a Gaussian distribution of $p_{0t}(\mathbf{x}_{t}\vert\mathbf{x}_{0})=\mathcal{N}(\mathbf{x}_{t}\vert\mu(t)\mathbf{x}_{0},\sigma^{2}(t)\mathbf{I})$, where its mean and covariance functions are characterized as a system of ODEs of
\begin{align}
&\frac{\diff \mu(t)}{\diff t}=-\frac{1}{2}\beta(t)\mu(t),\label{eq:mean}\\
&\frac{\diff \sigma^{2}(t)}{\diff t}=-\beta(t)\sigma^{2}(t)+g^{2}(t),\label{eq:covariance}
\end{align}
with initial conditions to be $\mu(0)=1$ and $\sigma^{2}(0)=0$.

Eq. \eqref{eq:mean} has its solution by
\begin{align*}
\mu(t)=e^{-\frac{1}{2}\int_{0}^{t}\beta(s)\diff s}.
\end{align*}
If we multiply $e^{\int_{0}^{t}\beta(s)\diff s}$ to Eq. \eqref{eq:covariance}, then Eq. \eqref{eq:covariance} equals to
\begin{eqnarray}
\lefteqn{e^{\int_{0}^{t}\beta(s)\diff s}\frac{\diff \sigma^{2}(t)}{\diff t}+e^{\int_{0}^{t}\beta(s)\diff s}\beta(t)\sigma^{2}(t)=e^{\int_{0}^{t}\beta(s)\diff s}g^{2}(t)}&\notag\\
&&\iff\frac{\diff \Big[e^{\int_{0}^{t}\beta(s)\diff s}\sigma^{2}(t)\Big]}{\diff t}=e^{\int_{0}^{t}\beta(s)\diff s}g^{2}(t)\notag\\
&&\iff e^{\int_{0}^{t}\beta(s)\diff s}\sigma^{2}(t)=\int_{0}^{t}e^{\int_{0}^{\tau}\beta(s)\diff s}g^{2}(\tau)\diff\tau+C\notag\\
&&\iff \sigma^{2}(t)=e^{-\int_{0}^{t}\beta(s)\diff s}\int_{0}^{t}e^{\int_{0}^{\tau}\beta(s)\diff s}g^{2}(\tau)\diff\tau + Ce^{-\int_{0}^{t}\beta(s)\diff s}.\label{eq:appendix_variance}
\end{eqnarray}
If we impose $\sigma^{2}(0)=0$ to Eq. \eqref{eq:appendix_variance}, then the constant $C$ satisfies $C=0$, and the variance formula becomes
\begin{align*}
\sigma^{2}(t)=e^{-\int_{0}^{t}\beta(s)\diff s}\int_{0}^{t}e^{\int_{0}^{\tau}\beta(s)\diff s}g^{2}(\tau)\diff\tau.
\end{align*}
To sum up, the family of linear SDEs of $\diff\mathbf{x}_{t}=-\frac{1}{2}\beta(t)\mathbf{x}_{t}\diff t+g(t)\diff\mathbf{w}_{t}$ gets the transition probability to be
\begin{align}\label{eq:appendix_transition_probability}
p_{0t}(\mathbf{x}_{t}\vert\mathbf{x}_{0})=\mathcal{N}\bigg(\mathbf{x}_{t}\Big\vert e^{-\frac{1}{2}\int_{0}^{t}\beta(s)\diff s}\mathbf{x}_{0},e^{-\int_{0}^{t}\beta(s)\diff s}\Big(\int_{0}^{t}e^{\int_{0}^{\tau}\beta(s)\diff s}g^{2}(\tau)\diff\tau\Big)\mathbf{I}\bigg).
\end{align}

\subsection{Diverging Denoising Loss}\label{sec:score_fail}

The gradient of the log transition probability, $\nabla\log{p_{0t}(\mathbf{x}_{t}\vert\mathbf{x}_{0})}=-\frac{\mathbf{x}_{t}-\mu(t)\mathbf{x}_{0}}{\sigma^{2}(t)}=-\frac{\mathbf{z}}{\sigma(t)}$, is diverging at $\mu(t)\mathbf{x}_{0}$, where $\mathbf{x}_{t}=\mu(t)\mathbf{x}_{0}+\sigma(t)\mathbf{z}$. Below Lemma \ref{lemma:2} indicates that $\Vert\mathbf{s}(\mathbf{x}_{t},t)-\nabla\log{p_{0t}(\mathbf{x}_{t}\vert\mathbf{x}_{0})}\Vert_{2}\rightarrow\infty$ for any continuous score function, $\mathbf{s}$. This leads that the denoising score loss diverges as $t\rightarrow 0$ as illustrated in Figure \ref{fig:nelbo}-(a).

\begin{lemma}\label{lemma:2}
		Let $\mathcal{H}_{[0,T]}=\{\mathbf{s}:\mathbb{R}^{d}\times [0,T]\rightarrow \mathbb{R}^{d},\text{ $\mathbf{s}$ is locally Lipschitz}\}$. Suppose a continuous vector field $\mathbf{v}$ defined on a subset $U$ of a compact manifold $M$ (i.e., $\mathbf{v}:U\subset M\rightarrow\mathbb{R}^{d}$) is unbounded, then there exists no $\mathbf{s}\in\mathcal{H}_{[0,T]}$ such that $\lim_{t\rightarrow 0}\mathbf{s}(\mathbf{x},t)=\mathbf{v}(\mathbf{x})$ a.e. on $U$.
	\end{lemma}

\begin{proof}[Proof of Lemma \ref{lemma:2}]
		Since $U$ is an open subset of a compact manifold $M$, $\Vert \mathbf{x}_{1}-\mathbf{x}_{2}\Vert\le \text{diam}(M)$ for all $\mathbf{x}_{1},\mathbf{x}_{2}\in U$. Also, if $t_{1},t_{2}\in [0,T]$, $\vert t_{1}-t_{2}\vert$ is bounded. Hence, the local Lipschitzness of $\mathbf{s}$ implies that there exists a positive $K>0$ such that $\Vert s(\mathbf{x}_{1},t_{1})-s(\mathbf{x}_{2},t_{2})\Vert \le K(\Vert \mathbf{x}_{1}-\mathbf{x}_{2}\Vert+\vert t_{1}-t_{2}\vert)$ for any $\mathbf{x}_{1},\mathbf{x}_{2}\in U$ and $t_{1},t_{2}\in[0,T]$. Therefore, for any $\mathbf{s}\in\mathcal{H}_{[0,T]}$, there exists $C>0$ such that $\Vert \mathbf{s}(\mathbf{x},t)\Vert<C$ for all $\mathbf{x}\in U$ and $t\in [0,T]$, which leads no $\mathbf{s}$ that satisfies $\mathbf{s}(\mathbf{x},t)\rightarrow v(\mathbf{x})$ a.e. on $U$ as $t\rightarrow 0$.
	\end{proof}

\subsection{General Weighted Diffusion Loss}\label{sec:general_weight}
The denoising score loss is
\begin{align}
\begin{split}\label{eq:appendix_denoising_loss}
\mathcal{L}(\bm{\theta};g^{2},\tau)=&\frac{1}{2}\int_{\tau}^{T}g^{2}(t)\mathbb{E}_{\mathbf{x}_{0},\mathbf{x}_{t}}\big[\Vert\mathbf{s}_{\bm{\theta}}(\mathbf{x}_{t},t)-\nabla_{\mathbf{x}_{t}}\log{p_{0t}(\mathbf{x}_{t}\vert\mathbf{x}_{0})}\Vert_{2}^{2}-\Vert\log{p_{0t}(\mathbf{x}_{t}\vert\mathbf{x}_{0})}\Vert_{2}^{2}\big]\diff t\\
&-\int_{\tau}^{T}\mathbb{E}_{\mathbf{x}_{t}}\big[\text{div}(\mathbf{f}(\mathbf{x}_{t},t))\big]\diff t-\mathbb{E}_{\mathbf{x}_{T}}\big[\log{\pi(\mathbf{x}_{T})}\big],
\end{split}
\end{align}
for any $\tau\in[0,T]$. For an appropriate class of function $A(t)$,
\begin{align*}
\begin{split}
\int_{0}^{T}\mathbb{P}(\tau)\bigg(\int_{\tau}^{T}A(t)\diff t\bigg)\diff \tau=&\int_{0}^{T}\int_{0}^{T}\mathbb{P}(\tau)A(t)1_{[\tau,T]}(t)\diff t\diff\tau\\
=&\int_{0}^{T}\int_{0}^{T}\mathbb{P}(\tau)A(t)1_{[\tau,T]}(t)\diff \tau\diff t\\
=&\int_{0}^{T}\int_{0}^{t}\mathbb{P}(\tau)A(t)\diff\tau\diff t\\
=&\int_{0}^{T}\bigg(\int_{0}^{t}\mathbb{P}(\tau)\diff\tau\bigg)A(t)\diff t
\end{split}
\end{align*}
holds by changing the order of integration. Therefore, we get
\begin{eqnarray*}
\lefteqn{\mathcal{L}_{ST}(\bm{\theta};g^{2},\mathbb{P}):=\mathbb{E}_{\mathbb{P}(\tau)}\big[\mathcal{L}(\bm{\theta};g^{2},\tau)\big]}\\
&&=\int_{0}^{T}\mathbb{P}(\tau)\bigg[\frac{1}{2}\int_{\tau}^{T}g^{2}(t)\mathbb{E}_{\mathbf{x}_{0},\mathbf{x}_{t}}\big[\Vert\mathbf{s}_{\bm{\theta}}(\mathbf{x}_{t},t)-\nabla_{\mathbf{x}_{t}}\log{p_{0t}(\mathbf{x}_{t}\vert\mathbf{x}_{0})}\Vert_{2}^{2}-\Vert\log{p_{0t}(\mathbf{x}_{t}\vert\mathbf{x}_{0})}\Vert_{2}^{2}\big]\diff t\\
&&\quad-\int_{\tau}^{T}\mathbb{E}_{\mathbf{x}_{t}}\big[\text{div}(\mathbf{f}(\mathbf{x}_{t},t))\big]\diff t-\mathbb{E}_{\mathbf{x}_{T}}\big[\log{\pi(\mathbf{x}_{T})}\big]\bigg]\diff\tau\\
&&=\int_{0}^{T}\Big(\int_{0}^{t}\mathbb{P}(\tau)\diff\tau\Big)\bigg[\frac{1}{2}g^{2}(t)\mathbb{E}_{\mathbf{x}_{0},\mathbf{x}_{t}}\big[\Vert\mathbf{s}_{\bm{\theta}}(\mathbf{x}_{t},t)-\nabla_{\mathbf{x}_{t}}\log{p_{0t}(\mathbf{x}_{t}\vert\mathbf{x}_{0})}\Vert_{2}^{2}-\Vert\log{p_{0t}(\mathbf{x}_{t}\vert\mathbf{x}_{0})}\Vert_{2}^{2}\big]\\
&&\quad-\mathbb{E}_{\mathbf{x}_{t}}\big[\text{div}(\mathbf{f}(\mathbf{x}_{t},t))\big]\bigg]\diff t-\mathbb{E}_{\mathbf{x}_{T}}\big[\log{\pi(\mathbf{x}_{T})}\big]\\
&&=\frac{1}{2}\int_{0}^{T}g_{\mathbb{P}}^{2}(t)\mathbb{E}_{\mathbf{x}_{0},\mathbf{x}_{t}}\big[\Vert\mathbf{s}_{\bm{\theta}}(\mathbf{x}_{t},t)-\nabla_{\mathbf{x}_{t}}\log{p_{0t}(\mathbf{x}_{t}\vert\mathbf{x}_{0})}\Vert_{2}^{2}\big]\diff t+C,
\end{eqnarray*}
where
\begin{align*}
C=-\frac{1}{2}\int_{0}^{T}g_{\mathbb{P}}^{2}(t)\mathbb{E}_{\mathbf{x}_{0},\mathbf{x}_{t}}\big[\Vert\log{p_{0t}(\mathbf{x}_{t}\vert\mathbf{x}_{0})}\Vert_{2}^{2}\big]\diff t-\int_{0}^{T}\Big(\int_{0}^{t}\mathbb{P}(\tau)\diff\tau\Big)\mathbb{E}_{\mathbf{x}_{t}}\big[\text{div}(\mathbf{f}(\mathbf{x}_{t},t))\big]\diff t-\mathbb{E}_{\mathbf{x}_{T}}\big[\log{\pi(\mathbf{x}_{T})}\big].
\end{align*}
If $\mathbf{f}(\mathbf{x}_{t},t)=-\frac{1}{2}\beta(t)\mathbf{x}_{t}$, then we have
\begin{align*}
C=-\frac{d}{2}\int_{0}^{T}\Big(\int_{0}^{t}\mathbb{P}(\tau)\diff\tau\Big)\frac{g^{2}(t)}{\sigma^{2}(t)}\diff t+\frac{d}{2}\int_{0}^{T}\Big(\int_{0}^{t}\mathbb{P}(\tau)\diff\tau\Big)\beta(t)\diff t-\mathbb{E}_{\mathbf{x}_{T}}\big[\log{\pi(\mathbf{x}_{T})}\big].
\end{align*}

\section{Theorems and Proofs}\label{sec:proof}

\begingroup
\renewcommand\thelemma{1}
\begin{lemma}\label{lemma:1}
For any $\tau\in[0,T]$,
\begin{align*}
\mathbb{E}_{\mathbf{x}_{\tau}}\big[-\log{p_{\tau}^{\bm{\theta}}(\mathbf{x}_{\tau})}\big]\le&\mathcal{L}(\bm{\theta};g^{2},\tau)=\frac{1}{2}\int_{\tau}^{T}g^{2}(t)\mathbb{E}_{\mathbf{x}_{0},\mathbf{x}_{t}}\big[\Vert\mathbf{s}_{\bm{\theta}}(\mathbf{x}_{t},t)-\nabla_{\mathbf{x}_{t}}\log{p_{0t}(\mathbf{x}_{t}\vert\mathbf{x}_{0})}\Vert_{2}^{2}\\
&-\Vert\nabla_{\mathbf{x}_{t}}\log{p_{0t}(\mathbf{x}_{t}\vert\mathbf{x}_{0})}\Vert_{2}^{2}\big]\diff t-\int_{\tau}^{T}\mathbb{E}_{\mathbf{x}_{t}}\big[\textup{div}(\mathbf{f}(\mathbf{x}_{t},t))\big]\diff t-\mathbb{E}_{\mathbf{x}_{T}}\big[\log{\pi(\mathbf{x}_{T})}\big].
\end{align*}
\end{lemma}
\endgroup

\begin{proof}
Suppose $\bm{\mu}$ is the path measure of the forward SDE, and $\bm{\nu}_{\bm{\theta}}$ is the path measure of the generative SDE. The restricted measure is defined by $\bm{\mu}\vert_{[\tau,T]}(\{F_{t}\}_{t=\tau}^{T}):=\bm{\mu}(\{F_{t}\}_{t=0}^{T})$, where $F_{t}=\mathbb{R}^{d}$ if $t\in[0,\tau)$ and $F_{t}$ is a measurable set in $\mathbb{R}^{d}$ otherwise. The restricted measure of $\bm{\nu}_{\bm{\theta}}$ is defined analogously. Then, by the data processing inequality, we get
\begin{align}\label{eq:appendix_data}
D_{KL}(p_{\tau}\Vert p_{\tau}^{\bm{\theta}})\le D_{KL}(\bm{\mu}\vert_{[\tau,T]}\Vert\bm{\nu}_{\bm{\theta}}\vert_{[\tau,T]}).
\end{align}
Now, from the chain rule of KL divergences, we have
\begin{align}\label{eq:appendix_chain}
D_{KL}(\bm{\mu}\vert_{[\tau,T]}\Vert\bm{\nu}_{\bm{\theta}}\vert_{[\tau,T]})=D_{KL}(p_{T}\Vert\pi)+\mathbb{E}_{\mathbf{z}\sim p_{T}}\Big[D_{KL}\big(\bm{\mu}\vert_{[\tau,T]}(\cdot\vert\mathbf{x}_{T}=\mathbf{z})\Vert\bm{\nu}_{\bm{\theta}}\vert_{[\tau,T]}(\cdot\vert\mathbf{x}_{T}=\mathbf{z})\big)\Big].
\end{align}
From the Girsanov theorem and the Martingale property, we get
\begin{align}\label{eq:appendix_girsanov}
D_{KL}\big(\bm{\mu}\vert_{[\tau,T]}(\cdot\vert\mathbf{x}_{T}=\mathbf{z})\Vert\bm{\nu}_{\bm{\theta}}\vert_{[\tau,T]}(\cdot\vert\mathbf{x}_{T}=\mathbf{z})\big)=\frac{1}{2}\int_{\tau}^{T}\mathbb{E}_{p_{t}(\mathbf{x}_{t})}\big[ g^{2}(t)\Vert\mathbf{s}_{\bm{\theta}}(\mathbf{x}_{t},t)-\nabla\log{p_{t}(\mathbf{x}_{t})}\Vert_{2}^{2}\big]\diff t,
\end{align}
and combining Eq. \eqref{eq:appendix_data}, \eqref{eq:appendix_chain} and \eqref{eq:appendix_girsanov}, we have
\begin{align}\label{eq:appendix_truncated_song}
D_{KL}(p_{\tau}\Vert p_{\tau}^{\bm{\theta}})\le D_{KL}(p_{T}\Vert \pi)+\frac{1}{2}\int_{\tau}^{T}\mathbb{E}_{p_{t}(\mathbf{x}_{t})}\big[ g^{2}(t)\Vert\mathbf{s}_{\bm{\theta}}(\mathbf{x}_{t},t)-\nabla\log{p_{t}(\mathbf{x}_{t})}\Vert_{2}^{2}\big]\diff t.
\end{align}
Now, from
\begin{eqnarray*}
\lefteqn{\frac{1}{2}\int_{\tau}^{T}\mathbb{E}_{p_{t}(\mathbf{x}_{t})}\big[ g^{2}(t)[\Vert\mathbf{s}_{\bm{\theta}}(\mathbf{x}_{t},t)-\nabla_{\mathbf{x}_{t}}\log{p_{t}(\mathbf{x}_{t})}\Vert_{2}^{2}-\Vert\log{p_{t}(\mathbf{x}_{t})}\Vert_{2}^{2}]\big]\diff t}&\\
&&=\frac{1}{2}\int_{\tau}^{T}\mathbb{E}_{p_{t}(\mathbf{x}_{t})}\big[ g^{2}(t)\Vert\mathbf{s}_{\bm{\theta}}(\mathbf{x}_{t},t)\Vert_{2}^{2}-2g^{2}(t)\mathbf{s}_{\bm{\theta}}(\mathbf{x}_{t},t)\cdot\nabla_{\mathbf{x}_{t}}\log{p_{t}(\mathbf{x}_{t})} \big]\diff t\\
&&=\frac{1}{2}\int_{\tau}^{T}\mathbb{E}_{p_{t}(\mathbf{x}_{t})}\big[ g^{2}(t)\Vert\mathbf{s}_{\bm{\theta}}(\mathbf{x}_{t},t)\Vert_{2}^{2}\big]\diff t-\int_{\tau}^{T}\int g^{2}(t)\mathbf{s}_{\bm{\theta}}(\mathbf{x}_{t},t)\cdot\nabla_{\mathbf{x}_{t}} p_{t}(\mathbf{x}_{t}) \diff \mathbf{x}_{t}\diff t\\
&&=\frac{1}{2}\int_{\tau}^{T}\mathbb{E}_{p_{t}(\mathbf{x}_{t})}\big[ g^{2}(t)\Vert\mathbf{s}_{\bm{\theta}}(\mathbf{x}_{t},t)\Vert_{2}^{2}\big]\diff t-\int_{\tau}^{T}\int g^{2}(t)\mathbf{s}_{\bm{\theta}}(\mathbf{x}_{t},t)\cdot\nabla_{\mathbf{x}_{t}} \int p_{r}(\mathbf{x}_{0})p_{0t}(\mathbf{x}_{t}\vert\mathbf{x}_{0})\diff\mathbf{x}_{0} \diff \mathbf{x}_{t}\diff t\\
&&=\frac{1}{2}\int_{\tau}^{T}\mathbb{E}_{p_{t}(\mathbf{x}_{t})}\big[ g^{2}(t)\Vert\mathbf{s}_{\bm{\theta}}(\mathbf{x}_{t},t)\Vert_{2}^{2}\big]\diff t-\int_{\tau}^{T}\int g^{2}(t)\mathbf{s}_{\bm{\theta}}(\mathbf{x}_{t},t)\cdot\int p_{r}(\mathbf{x}_{0})\nabla_{\mathbf{x}_{t}}p_{0t}(\mathbf{x}_{t}\vert\mathbf{x}_{0})\diff\mathbf{x}_{0} \diff \mathbf{x}_{t}\diff t\\
&&=\frac{1}{2}\int_{\tau}^{T}\mathbb{E}_{p_{r}(\mathbf{x}_{0})p_{0t}(\mathbf{x}_{t}\vert\mathbf{x}_{0})}\big[ g^{2}(t)[\Vert\mathbf{s}_{\bm{\theta}}(\mathbf{x}_{t},t)-\nabla_{\mathbf{x}_{t}}\log{p_{0t}(\mathbf{x}_{t}\vert\mathbf{x}_{0})}\Vert_{2}^{2}-\Vert\nabla_{\mathbf{x}_{t}}\log{p_{0t}(\mathbf{x}_{t}\vert\mathbf{x}_{0})}\Vert_{2}^{2}]\big]\diff t,
\end{eqnarray*}
we can transform $\Vert\mathbf{s}_{\bm{\theta}}(\mathbf{x}_{t},t)-\nabla\log{p_{t}(\mathbf{x}_{t})}\Vert_{2}^{2}$ into $\Vert\mathbf{s}_{\bm{\theta}}(\mathbf{x}_{t},t)-\nabla\log{p_{0t}(\mathbf{x}_{t}\vert\mathbf{x}_{0})}\Vert_{2}^{2}$, Eq. \eqref{eq:appendix_truncated_song} is equivalent to
\begin{eqnarray}\label{eq:appendix_truncated}
\lefteqn{\mathbb{E}_{p_{\tau}(\mathbf{x}_{\tau})}\big[-\log{p_{\tau}^{\bm{\theta}}(\mathbf{x}_{\tau})}\big]\le D_{KL}(p_{T}\Vert \pi)+\frac{1}{2}\int_{\tau}^{T}\mathbb{E}_{p_{t}(\mathbf{x}_{t})}\big[ g^{2}(t)\Vert\mathbf{s}_{\bm{\theta}}(\mathbf{x}_{t},t)-\nabla\log{p_{t}(\mathbf{x}_{t})}\Vert_{2}^{2}\big]\diff t+\mathcal{H}(p_{\tau})}&\\
&&=D_{KL}(p_{T}\Vert \pi)+\frac{1}{2}\int_{\tau}^{T}\mathbb{E}_{p_{t}(\mathbf{x}_{t})}\big[ g^{2}(t)\Vert\mathbf{s}_{\bm{\theta}}(\mathbf{x}_{t},t)-\nabla\log{p_{0t}(\mathbf{x}_{t}\vert\mathbf{x}_{0})}\Vert_{2}^{2}-\Vert\nabla\log{p_{0t}(\mathbf{x}_{t}\vert\mathbf{x}_{0})}\Vert_{2}^{2}\big]\diff t\\
&&\quad+\frac{1}{2}\int_{\tau}^{T}\mathbb{E}_{p_{t}(\mathbf{x}_{t})}\big[g^{2}(t)\nabla\log{p_{t}(\mathbf{x}_{t})}\Vert_{2}^{2}\big]\diff t+\mathcal{H}(p_{\tau}).
\end{eqnarray}
Now, directly applying Theorem 4 of \citet{song2021maximum}, the entropy of $\mathcal{H}(p_{\tau})$ becomes
\begin{align}\label{eq:appendix_theorem4_song}
\mathcal{H}(p_{\tau})=\mathcal{H}(p_{T})-\frac{1}{2}\int_{\tau}^{T}\mathbb{E}_{p_{t}(\mathbf{x}_{t})}\big[2\text{div}\big(\mathbf{f}(\mathbf{x}_{t},t)\big)+g^{2}(t)\Vert\nabla\log{p_{t}(\mathbf{x}_{t})}\Vert_{2}^{2}\big]\diff t.
\end{align}
Therefore, from Eq. \eqref{eq:appendix_truncated} and \eqref{eq:appendix_theorem4_song}, we get
\begin{align*}
\mathbb{E}_{p_{\tau}(\mathbf{x}_{\tau})}\big[-\log{p_{\tau}^{\bm{\theta}}(\mathbf{x}_{\tau})}\big]\le& \frac{1}{2}\int_{\tau}^{T}\mathbb{E}_{p_{t}(\mathbf{x}_{t})}\big[ g^{2}(t)\Vert\mathbf{s}_{\bm{\theta}}(\mathbf{x}_{t},t)-\nabla\log{p_{0t}(\mathbf{x}_{t}\vert\mathbf{x}_{0})}\Vert_{2}^{2}-\Vert\nabla\log{p_{0t}(\mathbf{x}_{t}\vert\mathbf{x}_{0})}\Vert_{2}^{2}\big]\diff t\\
&-\int_{\tau}^{T}\mathbb{E}_{\mathbf{x}_{t}}\big[\textup{div}(\mathbf{f}(\mathbf{x}_{t},t))\big]\diff t-\mathbb{E}_{\mathbf{x}_{T}}\big[\log{\pi(\mathbf{x}_{T})}\big].
\end{align*}
\end{proof}

\begingroup
	\renewcommand\thetheorem{1}
\begin{theorem}
Suppose $\lambda(t)$ is a weighting function of the NCSN loss. If $\frac{\lambda(t)}{g^{2}(t)}$ is a nondecreasing and nonnegative absolutely continuous function on $[\epsilon,T]$ and zero on $[0,\epsilon)$, then
\begin{align*}
\mathcal{L}(\bm{\theta};\lambda,\epsilon)\ge&\int_{\epsilon}^{T}\Big(\frac{\lambda(\tau)}{g^{2}(\tau)}\Big)'\mathbb{E}_{\mathbf{x}_{\tau}}\big[-\log{p_{\tau}^{\bm{\theta}}(\mathbf{x}_{\tau})}\big]\diff \tau+\frac{\lambda(\epsilon)}{g^{2}(\epsilon)}\mathbb{E}_{\mathbf{x}_{\epsilon}}\big[-\log{p_{\epsilon}^{\bm{\theta}}(\mathbf{x}_{\epsilon})}\big]\\
&+\int_{\epsilon}^{T}\Big(\frac{\lambda(\tau)}{g^{2}(\tau)}-1\Big)\mathbb{E}_{\mathbf{x}_{\tau}}\big[\textup{div}(\mathbf{f}(\mathbf{x}_{\tau},\tau))\big]\diff \tau+\Big[\frac{\lambda(T)}{g^{2}(T)}-1\Big]\mathbb{E}_{\mathbf{x}_{T}}\big[\log{\pi(\mathbf{x}_{T})}\big].
\end{align*}
\end{theorem}
\endgroup

\begin{proof}
We prove the theorm by using
\begin{align}\label{eq:exchange}
\begin{split}
\int_{\epsilon}^{T}\lambda(t)A(t)\diff t=&\int_{\epsilon}^{T}\bigg[\int_{\epsilon}^{t}\Big(\frac{\lambda(t)}{g^{2}(t)}\Big)'\diff\tau+\frac{\lambda(\epsilon)}{g^{2}(\epsilon)}\bigg]g^{2}(t)A(t)\diff t\\
=&\int_{\epsilon}^{T}\int_{\epsilon}^{T}1_{[\epsilon,t]}(\tau)\Big(\frac{\lambda(\tau)}{g^{2}(\tau)}\Big)' g^{2}(t)A(t)\diff \tau\diff t+\frac{\lambda(\epsilon)}{g^{2}(\epsilon)}\int_{\epsilon}^{T}g^{2}(t)A(t)\diff t\\
=&\int_{\epsilon}^{T}\Big(\frac{\lambda(\tau)}{g^{2}(\tau)}\Big)'\int_{\tau}^{T}g^{2}(t)A(t)\diff t \diff\tau+\frac{\lambda(\epsilon)}{g^{2}(\epsilon)}\int_{\epsilon}^{T}g^{2}(t)A(t)\diff t.
\end{split}
\end{align}
By plugging $A(t)=\frac{1}{2}\mathbb{E}_{\mathbf{x}_{t}}\big[\Vert\mathbf{s}_{\bm{\theta}}(\mathbf{x}_{t},t)-\nabla_{\mathbf{x}_{t}}\log{p_{t}(\mathbf{x}_{t})}\Vert_{2}^{2}-\Vert\log{p_{t}(\mathbf{x}_{t})}\Vert_{2}^{2}\big]$ in Eq. \eqref{eq:exchange}, we have
\begin{align}
\mathcal{L}(\bm{\theta};\lambda,\epsilon):=&\frac{1}{2}\int_{\epsilon}^{T}\lambda(t)\mathbb{E}_{\mathbf{x}_{t}}\big[\Vert\mathbf{s}_{\bm{\theta}}(\mathbf{x}_{t},t)-\nabla_{\mathbf{x}_{t}}\log{p_{t}(\mathbf{x}_{t})}\Vert_{2}^{2}-\Vert\nabla_{\mathbf{x}_{t}}\log{p_{t}(\mathbf{x}_{t})}\Vert_{2}^{2}\big]\diff t\notag\\
&-\int_{\epsilon}^{T}\mathbb{E}_{\mathbf{x}_{t}}\big[\text{div}(\mathbf{f}(\mathbf{x}_{t},t))\big]\diff t-\mathbb{E}_{\mathbf{x}_{T}}\big[\log{\pi(\mathbf{x}_{T})}\big]\notag\\
=&\int_{\epsilon}^{T}\Big(\frac{\lambda(\tau)}{g^{2}(\tau)}\Big)'\bigg[\frac{1}{2}\int_{\tau}^{T}g^{2}(t)\mathbb{E}_{\mathbf{x}_{t}}\big[\Vert\mathbf{s}_{\bm{\theta}}(\mathbf{x}_{t},t)-\nabla_{\mathbf{x}_{t}}\log{p_{t}(\mathbf{x}_{t})}\Vert_{2}^{2}-\Vert\nabla_{\mathbf{x}_{t}}\log{p_{t}(\mathbf{x}_{t})}\Vert_{2}^{2}\big]\diff t\notag\\
&\quad\quad-\int_{\tau}^{T}\mathbb{E}_{\mathbf{x}_{t}}\big[\text{div}(\mathbf{f}(\mathbf{x}_{t},t))\big]\diff t-\mathbb{E}_{\mathbf{x}_{T}}\big[\log{\pi(\mathbf{x}_{T})}\big]\bigg]\diff \tau\notag\\
&+\frac{\lambda(\epsilon)}{g^{2}(\epsilon)}\bigg[\frac{1}{2}\int_{\epsilon}^{T}g^{2}(t)\mathbb{E}_{\mathbf{x}_{t}}\big[\Vert\mathbf{s}_{\bm{\theta}}(\mathbf{x}_{t},t)-\nabla_{\mathbf{x}_{t}}\log{p_{t}(\mathbf{x}_{t})}\Vert_{2}^{2}-\Vert\nabla_{\mathbf{x}_{t}}\log{p_{t}(\mathbf{x}_{t})}\Vert_{2}^{2}\big]\diff t\label{eq:appendix_derivation}\\
&\quad\quad-\int_{\epsilon}^{T}\mathbb{E}_{\mathbf{x}_{t}}\big[\text{div}(\mathbf{f}(\mathbf{x}_{t},t))\big]\diff t-\mathbb{E}_{\mathbf{x}_{T}}\big[\log{\pi(\mathbf{x}_{T})}\big]\bigg]\notag\\
&+\int_{\epsilon}^{T}\Big(\frac{\lambda(\tau)}{g^{2}(\tau)}\Big)'\int_{\tau}^{T}\mathbb{E}_{\mathbf{x}_{t}}\big[\text{div}(\mathbf{f}(\mathbf{x}_{t},t))\big]\diff t\diff\tau+\Big(\frac{\lambda(\epsilon)}{g^{2}(\epsilon)}\Big)\int_{\epsilon}^{T}\mathbb{E}_{\mathbf{x}_{t}}\big[\text{div}(\mathbf{f}(\mathbf{x}_{t},t))\big]\diff t\notag\\
&-\int_{\epsilon}^{T}\mathbb{E}_{\mathbf{x}_{t}}\big[\text{div}(\mathbf{f}(\mathbf{x}_{t},t))\big]\diff t+\mathbb{E}_{\mathbf{x}_{T}}\big[\log{\pi(\mathbf{x}_{T})}\big]\bigg[\int_{\epsilon}^{T}\Big(\frac{\lambda(\tau)}{g^{2}(\tau)}\Big)'\diff\tau+\frac{\lambda(\epsilon)}{g^{2}(\epsilon)}-1\bigg].\notag
\end{align}

Also, plugging $A(t)=\frac{1}{g^{2}(t)}\mathbb{E}_{\mathbf{x}_{t}}\big[\text{div}\big(\mathbf{f}(\mathbf{x}_{t},t)\big)\big]$ into Eq. \eqref{eq:exchange}, we have
\begin{align}\label{eq:appendix_plugging}
\int_{\epsilon}^{T}\frac{\lambda(t)}{g^{2}(t)}\mathbb{E}_{\mathbf{x}_{t}}\big[\text{div}\big(\mathbf{f}(\mathbf{x}_{t},t)\big)\big]=\int_{\epsilon}^{T}\Big(\frac{\lambda(\tau)}{g^{2}(\tau)}\Big)'\int_{\tau}^{T}\mathbb{E}_{\mathbf{x}_{t}}\big[\text{div}(\mathbf{f}(\mathbf{x}_{t},t))\big]\diff t\diff\tau+\Big(\frac{\lambda(\epsilon)}{g^{2}(\epsilon)}\Big)\int_{\epsilon}^{T}\mathbb{E}_{\mathbf{x}_{t}}\big[\text{div}(\mathbf{f}(\mathbf{x}_{t},t))\big]\diff t.
\end{align}

Using Eq. \eqref{eq:appendix_derivation} and \eqref{eq:appendix_plugging}, we get
\begin{align}\label{eq:appendix_final_}
\begin{split}
\mathcal{L}(\bm{\theta};\lambda,\epsilon)=&\int_{\epsilon}^{T}\Big(\frac{\lambda(\tau)}{g^{2}(\tau)}\Big)'\mathcal{L}(\bm{\theta};g^{2},\tau)\diff\tau+\frac{\lambda(\epsilon)}{g^{2}(\epsilon)}\mathcal{L}(\bm{\theta};g^{2},\epsilon)\\
&+\int_{\epsilon}^{T}\Big(\frac{\lambda(t)}{g^{2}(t)}-1\Big)\mathbb{E}_{\mathbf{x}_{t}}\big[\text{div}(\mathbf{f}(\mathbf{x}_{t},t))\big]\diff t+\Big[\frac{\lambda(T)}{g^{2}(T)}-1\Big]\mathbb{E}_{\mathbf{x}_{T}}\big[\log{\pi(\mathbf{x}_{T})}\big].
\end{split}
\end{align}
Then, applying Lemma \ref{lemma:1} to Eq. \eqref{eq:appendix_final_} yields the desired result.
\end{proof}

\begin{corollary}
Suppose $\lambda(t)$ is a weighting function of the NCSN loss. If $\frac{\lambda(t)}{g^{2}(t)}$ is a nondecreasing and nonnegative continuous function on $[\epsilon,T]$ and zero on $[0,\epsilon)$, then
\begin{align*}
&\frac{1}{2}\int_{\epsilon}^{T}\lambda(t)\mathbb{E}_{\mathbf{x}_{t}}\big[\Vert\mathbf{s}_{\bm{\theta}}(\mathbf{x}_{t},t)-\nabla_{\mathbf{x}_{t}}\log{p_{t}(\mathbf{x}_{t})}\Vert_{2}^{2}\big]\diff t+\frac{\lambda(T)}{g^{2}(T)}D_{KL}(p_{T}\Vert\pi)\\
&\quad\quad\quad\ge\int_{\epsilon}^{T}\Big(\frac{\lambda(\tau)}{g^{2}(\tau)}\Big)'D_{KL}(p_{\tau}\Vert p_{\tau}^{\bm{\theta}})\diff \tau+\frac{\lambda(\epsilon)}{g^{2}(\epsilon)}D_{KL}(p_{\epsilon}\Vert p_{\epsilon}^{\bm{\theta}}).
\end{align*}
\end{corollary}

\begin{remark}
A direct extension of the proof indicates that Theorem \ref{thm:1} still holds when $\frac{\lambda(t)}{g^{2}(t)}$ has finite jump on $[0,T]$.
\end{remark}
\begin{remark}
The weight of $\frac{\lambda(T)}{g^{2}(T)}$ is the normalizing constant of the unnormalized truncation probability, $\mathbb{P}$.
\end{remark}

\begin{proof}
By plugging $A(t)=\frac{1}{2}\mathbb{E}_{\mathbf{x}_{t}}\big[\Vert\mathbf{s}_{\bm{\theta}}(\mathbf{x}_{t},t)-\nabla_{\mathbf{x}_{t}}\log{p_{t}(\mathbf{x}_{t})}\Vert_{2}^{2}\big]$ in Eq. \eqref{eq:exchange} and using Lemma \ref{lemma:1}, we have
\begin{eqnarray*}
\lefteqn{\frac{1}{2}\int_{\epsilon}^{T}\lambda(t)\mathbb{E}_{\mathbf{x}_{t}}\big[\Vert\mathbf{s}_{\bm{\theta}}(\mathbf{x}_{t},t)-\nabla_{\mathbf{x}_{t}}\log{p_{t}(\mathbf{x}_{t})}\Vert_{2}^{2}\big]\diff t+\frac{\lambda(T)}{g^{2}(T)}D_{KL}(p_{T}\Vert\pi)}&\\
&&=\int_{\epsilon}^{T}\Big(\frac{\lambda(\tau)}{g^{2}(\tau)}\Big)'\frac{1}{2}\int_{\tau}^{T}g^{2}(t)\mathbb{E}_{\mathbf{x}_{t}}\big[\Vert\mathbf{s}_{\bm{\theta}}(\mathbf{x}_{t},t)-\nabla_{\mathbf{x}_{t}}\log{p_{t}(\mathbf{x}_{t})}\Vert_{2}^{2}\big]\diff t\diff\tau\\
&&\quad+\Big(\frac{\lambda(\epsilon)}{g^{2}(\epsilon)}\Big)\frac{1}{2}\int_{\epsilon}^{T}g^{2}(t)\mathbb{E}_{\mathbf{x}_{t}}\big[\Vert\mathbf{s}_{\bm{\theta}}(\mathbf{x}_{t},t)-\nabla_{\mathbf{x}_{t}}\log{p_{t}(\mathbf{x}_{t})}\Vert_{2}^{2}\big]\diff t+\frac{\lambda(T)}{g^{2}(T)}D_{KL}(p_{T}\Vert\pi)\\
&&\ge\int_{\epsilon}^{T}\Big(\frac{\lambda(\tau)}{g^{2}(\tau)}\Big)'\big[D_{KL}(p_{\tau}\Vert p_{\tau}^{\bm{\theta}})-D_{KL}(p_{T}\Vert\pi)\big]\diff\tau+\frac{\lambda(\epsilon)}{g^{2}(\epsilon)}\big[D_{KL}(p_{\epsilon}\Vert p_{\epsilon}^{\bm{\theta}})-D_{KL}(p_{T}\Vert\pi)\big]+\frac{\lambda(T)}{g^{2}(T)}D_{KL}(p_{T}\Vert\pi)\\
&&=\int_{\epsilon}^{T}\Big(\frac{\lambda(\tau)}{g^{2}(\tau)}\Big)'D_{KL}(p_{\tau}\Vert p_{\tau}^{\bm{\theta}})\diff\tau+\frac{\lambda(\epsilon)}{g^{2}(\epsilon)}D_{KL}(p_{\epsilon}\Vert p_{\epsilon}^{\bm{\theta}}).
\end{eqnarray*}
\end{proof}

\section{Additional Score Architectures and SDEs}

\subsection{Additional Score Architectures: Unbounded Parametrization}

From the released code of \citet{song2020score}, the NCSN++ network is modeled by $\mathbf{s}_{\bm{\theta}}(\mathbf{x}_{t},\log{\sigma(t)})$, where the second argument is $\log{\sigma(t)}$ instead of $t$. Experiments with $\mathbf{s}_{\bm{\theta}}(\mathbf{x}_{t},t)$ or $\mathbf{s}_{\bm{\theta}}(\mathbf{x}_{t},\sigma(t))$ were not as good as the parametrization of $\mathbf{s}_{\bm{\theta}}(\mathbf{x}_{t},\log{\sigma(t)})$, and we analyze this experimental results from Lemma \ref{lemma:2} and Proposition \ref{prop:1}.

\begin{proposition}\label{prop:1}
		Let $\mathcal{H}_{[1,\infty)}=\{\mathbf{s}:\mathbb{R}^{d}\times[1,\infty)\rightarrow\mathbb{R}^{d},\text{ $\mathbf{s}$ is locally Lipschitz}\}$. Suppose a continuous vector field $\mathbf{v}$ defined on a $d$-dimensional open subset $U$ of a compact manifold $M$ is unbounded, and the projection of $\mathbf{v}$ on each axis is locally integrable. Then, there exists $\mathbf{s}\in\mathcal{H}_{[1,\infty)}$ such that $\lim_{\eta\rightarrow \infty}\mathbf{s}(\mathbf{x},\eta)=\mathbf{v}(\mathbf{x})$ a.e. on $U$.
	\end{proposition}

The gradient of the log transition probability diverges at $t\approx 0$ theoretically (Section \ref{sec:score_fail}) and empirically (Figure \ref{fig:uddpm}-(a)). Here, in high-dimensional space, $p_{0t}(\mathbf{x}_{t}\vert\mathbf{x}_{0})/p_{0t}(\mathbf{x}_{t}\vert\mathbf{x}_{0})$ with $\mathbf{x}_{0}\neq\mathbf{x}_{0}'$ is either zero or infinity. Thus, the data score is nearly identical to the gradient of the log transition probability, $\Vert\nabla_{\mathbf{x}_{t}}\log{p_{t}(\mathbf{x}_{t})}\Vert_{2}^{2}=\Vert\nabla_{\mathbf{x}_{t}}\log{\int p_{r}(\mathbf{x}_{0})p_{0t}(\mathbf{x}_{t}\vert\mathbf{x}_{0})\diff\mathbf{x}_{0}}\Vert_{2}^{2}\approx\Vert\nabla_{\mathbf{x}_{t}}\log{p_{0t}(\mathbf{x}_{t}\vert\mathbf{x}_{0})}\Vert_{2}^{2}$, and the observation of Figure \ref{fig:uddpm}-(a) is valid for the exact data score, as well.

Although Lemma \ref{lemma:2} is based on $\mathbf{s}_{\bm{\theta}}(\mathbf{x}_{t},t)$, the identical result also holds for the parametrization of $\mathbf{s}_{\bm{\theta}}(\mathbf{x}_{t},\sigma(t))$, so it indicates that both $\mathbf{s}_{\bm{\theta}}(\mathbf{x}_{t},t)$ and $\mathbf{s}_{\bm{\theta}}(\mathbf{x}_{t},\sigma(t))$ cannot estimate the data score as $t\rightarrow 0$. On the other hand, Proposition \ref{prop:1} implies that there exists a score function that estimates the unbounded data score asymptotically, and Proposition \ref{prop:1} explains the reason why the parametrization of \citet{song2020score}, i.e., $\mathbf{s}_{\bm{\theta}}(\mathbf{x}_{t},\log{\sigma(t)})$, is successful on score estimation.

On top of that, we introduce another parametrization that particularly focuses on the score estimation near $t\approx 0$. We name Unbounded NCSN++ (UNCSN++) as the network of $\mathbf{s}_{\bm{\theta}}(\mathbf{x}_{t},\eta(t))$ with $\eta(t)=\left\{\begin{array}{ll}\log{\sigma(t)}&\text{if }\sigma(t)\ge \sigma_{0}\\ -\frac{c_{1}}{\sigma(t)}+c_{2}&\text{if }\sigma(t)<\sigma_{0}\end{array}\right.$ and Unbounded DDPM++ (UDDPM++) as the network of $\mathbf{s}_{\bm{\theta}}(\mathbf{x}_{t},\eta(t))$ with $\eta(t):=\int\frac{g^{2}(t)}{\sigma^{2}(t)}\diff t$.

	In UNCSN++, $c_{1}, c_{2}$ and $\sigma_{0}$ are the hyperparameters. By acknowledging the parametrization of $\log{\sigma(t)}$, we choose $\sigma_{0}$ as $0.01$. Also, to satisfy the continuously differentiability of $\eta(t)$, two hyperparameters $c_{1}$ and $c_{2}$ satisfy a system of equations with degree 2, so $c_{1}$ and $c_{2}$ are fully determined with this system of equations.
	
\begin{figure*}[t]
\centering
	\begin{subfigure}{0.32\linewidth}
	\includegraphics[width=\linewidth]{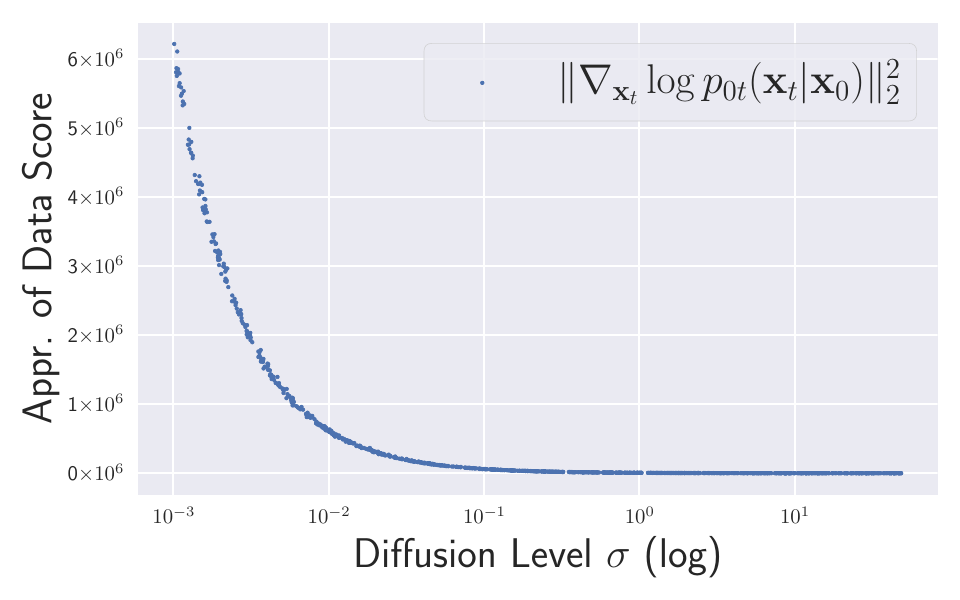}
	\subcaption{Approximate data score diverges.}
	\end{subfigure}
	\hfil
	\begin{subfigure}{0.32\linewidth}
	\includegraphics[width=\linewidth]{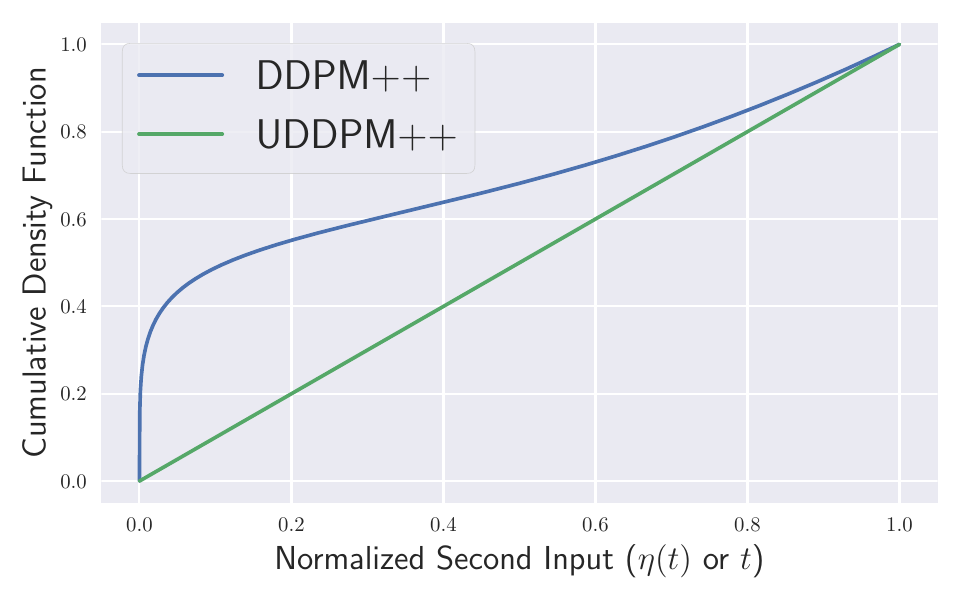}
	\subcaption{Cumulative density function of $t$ and $\eta$.}
	\end{subfigure}
	\hfil
	\begin{subfigure}{0.32\linewidth}
	\includegraphics[width=\linewidth]{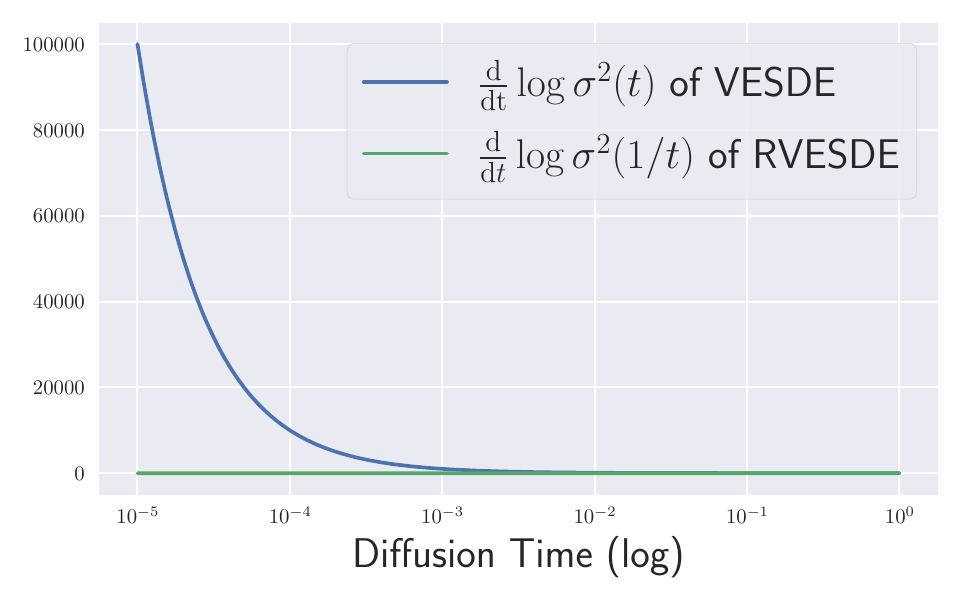}
	\subcaption{VESDE violates geometric progression.}
	\end{subfigure}
	\caption{(a) The approximate data score, $\Vert\nabla_{\mathbf{x}_{t}}\log{p_{t}(\mathbf{x}_{t})}\Vert_{2}^{2}=\Vert\nabla_{\mathbf{x}_{t}}\log{\int p_{r}(\mathbf{x}_{0})p_{0t}(\mathbf{x}_{t}\vert\mathbf{x}_{0})\diff\mathbf{x}_{0}}\Vert_{2}^{2}\approx\Vert\nabla_{\mathbf{x}_{t}}\log{p_{0t}(\mathbf{x}_{t}\vert\mathbf{x}_{0})}\Vert_{2}^{2}$, diverges as $t\rightarrow 0$. (b) Comparison of DDPM++ and UDDPM++ in terms of the cumulative density function of the second input. (c) Comparison of VESDE and RVESDE in terms of $\frac{\diff}{\diff t}\log{\sigma^{2}}$.}
	\label{fig:uddpm}
\end{figure*}
The choice of such $\eta(t)$ for UDDPM++ is expected to enhance the score estimation near $t\approx 0$ because the input of $\eta(t)$ is distributed uniformly when we draw samples from the importance weight. Concretely, when the sampling distribution on the diffusion time is given by $p_{iw}(t)\propto\frac{g^{2}(t)}{\sigma^{2}(t)}$, the $\eta$-distribution from the importance sampling becomes $p(\eta)\propto 1$, which is depicted in Figure \ref{fig:uddpm}-(b).

	\begin{proof}[Proof of Proposition \ref{prop:1}]
		Let $h$ be a standard mollifier function. If $h_{t}(x)=t^{-n}h(\mathbf{x}/t)$, then $v_{t}:=h_{t}*v$ converges to $v$ a.e. on $U$ as $t\rightarrow 0$ (Theorem 7-(ii) of Appendix C in \cite{evans1998partial}). Therefore, if we define $s(\mathbf{x},\eta):=v_{1/\eta}(\mathbf{x})$ on the domain of $v_{1/\eta}(\mathbf{x})$ and $s(\mathbf{x},\eta):=0$ elsewhere, then $s(\mathbf{x},\eta)=v_{1/\eta}(\mathbf{x})\rightarrow v(\mathbf{x})$ a.e. on $U$ as $\eta\rightarrow\infty$.
		
		Now, to show that $\mathbf{s}(\mathbf{x},\eta)$ is locally Lipschitz, let $\tilde{M}\times [\underline{\eta},\overline{\eta}]$ be a compact subset of $\mathbb{R}^{n}\times[1,\infty)$. From $\Vert \mathbf{s}(\mathbf{x}_{1},\eta_{1})-\mathbf{s}(\mathbf{x}_{2},\eta_{2})\Vert=\Vert v_{1/\eta_{1}}(\mathbf{x}_{1})-v_{1/\eta_{2}}(\mathbf{x}_{2})\Vert\le\Vert v_{1/\eta_{1}}(\mathbf{x}_{1})-v_{1/\eta_{1}}(\mathbf{x}_{2})\Vert+\Vert v_{1/\eta_{1}}(\mathbf{x}_{2})-v_{1/\eta_{2}}(\mathbf{x}_{2})\Vert$, if there exists $K_{1},K_{2}>0$ such that $\Vert v_{1/\eta_{1}}(\mathbf{x}_{1})-v_{1/\eta_{1}}(\mathbf{x}_{2})\Vert\le K_{1}\Vert \mathbf{x}_{1}-\mathbf{x}_{2}\Vert$ and $\Vert v_{1/\eta_{1}}(\mathbf{x}_{1})-v_{1/\eta_{2}}(\mathbf{x}_{1})\Vert\le K_{2}\vert \eta_{1}-\eta_{2}\vert$ for all $\mathbf{x}_{1},\mathbf{x}_{2}\in \tilde{M}$ and $\eta_{1},\eta_{2}\in [\underline{\eta},\overline{\eta}]$, then $\mathbf{s}(\mathbf{x},\eta)=v_{1/\eta}(\mathbf{x})$ is Lipschitz on $\tilde{M}\times[\underline{\eta},\overline{\eta}]$. 
		
		First, since $v_{1/\eta}$ is infinitely differentiable on its domain (Theorem 7-(i) of Appendix C in \cite{evans1998partial}) and $\eta\in[\underline{\eta},\overline{\eta}]$, there exists $K_{1}>0$ such that $\Vert v_{1/\eta}(\mathbf{x}_{1})-v_{1/\eta}(\mathbf{x}_{2})\Vert\le K_{1}\Vert \mathbf{x}_{1}-\mathbf{x}_{2}\Vert$. Second, the mollifier satisfies the uniform convergence on any compact subset of $U$ (Theorem 7-(iii) of Appendix C in \cite{evans1998partial}), which leads that $\Vert v_{1/\eta_{1}}(\mathbf{x})-v_{1/\eta_{2}}(\mathbf{x})\Vert\le K_{2}\vert \frac{1}{\eta_{1}}-\frac{1}{\eta_{2}}\vert=K_{2}\frac{\vert \eta_{1}-\eta_{2}\vert}{\eta_{1}\eta_{2}}\le K_{3}\vert \eta_{1}-\eta_{2}\vert$ for some $K_{2},K_{3}>0$. Therefore, $\mathbf{s}$ becomes an element of $\mathcal{H}_{[1,\infty)}$.
	\end{proof}	

\subsection{Additional SDE: Reciprocal VESDE}

VESDE assumes $g(t)=\sigma_{min}(\frac{\sigma_{max}}{\sigma_{min}})^{t}\sqrt{2\log{\frac{\sigma_{max}}{\sigma_{min}}}}$. Then, the variance of the transition probability $p_{0t}(\mathbf{x}_{t}\vert\mu_{VE}(t)\mathbf{x}_{0},\sigma_{VE}^{2}(t))$ becomes $\sigma_{VE}^{2}(t)=\int_{0}^{t}g^{2}(s)\diff s =\sigma_{min}^{2}[(\frac{\sigma_{max}}{\sigma_{min}})^{2t}-1]$ if the diffusion starts from $t=0$ with the initial condition of $\mathbf{x}_{0}\sim p_{r}$. VESDE was originally introduced in \citet{song2020improved} in order to satisfy the geometric property for its smooth transition of the distributional shift. Mathematically, the variance is geometric if $\frac{\diff}{\diff t}\log{\sigma_{VE}^{2}(t)}$ is a constant, but VESDE losses the geometric property as illustrated in Figure \ref{fig:uddpm}-(c).

To attain the geometric property in VESDE, VESDE approximates the variance to be $\tilde{\sigma}_{VE}^{2}(t)=\sigma_{min}^{2}(\frac{\sigma_{max}}{\sigma_{min}})^{2t}$ by omitting 1 from $\sigma_{VE}^{2}(t)$. However, this approximation leads that $\mathbf{x}_{t}$ is not converging to $\mathbf{x}_{0}$ in distribution because $\sigma_{min}^{2}(\frac{\sigma_{max}}{\sigma_{min}})^{2t}\rightarrow \sigma_{min}^{2}\neq 0$ as $t\rightarrow 0$. Indeed, a bit stronger claim is possible:
\begin{proposition}\label{prop:2}
		There is no SDE that has the stochastic process $\{\mathbf{x}_{t}\}_{t\in[0,T]}$, defined by a transition probability $p_{0t}(\mathbf{x}_{t}\vert\mathbf{x}_{0})=\mathcal{N}(\mathbf{x}_{t};\mathbf{x}_{0},\sigma_{min}^{2}(\frac{\sigma_{max}}{\sigma_{min}})^{2t}\mathbf{I})$, as the solution.
	\end{proposition}
	Proposition \ref{prop:2} indicates that if we approximate the variance by $\sigma_{VE}^{2}(t)$, then the reverse diffusion process cannot be modeled by a generative process. 

Rigorously, however, if the diffusion process starts from $t=-\infty$, rather than $t=0$, then the variance of the transition probability becomes $\sigma_{VE,-\infty}^{2}(t)=\int_{-\infty}^{t}g^{2}(s)\diff s=\sigma_{min}^{2}(\frac{\sigma_{max}}{\sigma_{min}})^{2t}$, which is exactly the variance $\tilde{\sigma}_{VE}^{2}(t)$. Therefore, VESDE can be considered as a diffusion process starting from $t=-\infty$.

From this point of view, we introduce a SDE that satisfies the geometric progression property starting from $t=0$. We name a new SDE as the Reciprocal VE SDE (RVESDE). RVESDE has the identical form of SDE, $\diff\mathbf{x}_{t}=g_{RVE}(t)\diff\mathbf{w}_{t}$, with
\begin{align*}
	g_{RVE}(t):=\left \{\begin{array}{ll}
	\sigma_{max}\big(\frac{\sigma_{min}}{\sigma_{max}})^{\frac{\epsilon}{t}}\frac{\sqrt{2\epsilon\log{(\frac{\sigma_{max}}{\sigma_{min}})}}}{t} & \text{if }t>0,\\
	0 & \text{if }t=0.
	\end{array}
	\right.
	\end{align*}
Then, the transition probability of RVESDE becomes
\begin{align*}
p_{0t}(\mathbf{x}_{t}\vert\mathbf{x}_{0})=\mathcal{N}\bigg(\mathbf{x}_{t};\mathbf{x}_{0},\sigma_{max}^{2}\Big(\frac{\sigma_{min}}{\sigma_{max}}\Big)^{\frac{2\epsilon}{t}}\mathbf{I}\bigg).
\end{align*}
As illustrated in Figure \ref{fig:uddpm}-(c), RVESDE attains the geometric property at the expense of having reciprocated time, $1/t$. Also, RVESDE satisfies $\sigma_{RVE}^{2}(\epsilon)=\sigma_{min}^{2}$ and $\sigma_{RVE}^{2}(T)\approx\sigma_{max}^{2}$. The existence and uniqueness of solution for RVESDE is guaranteed by Theorem 5.2.1 in \cite{oksendal2013stochastic}. 

\section{Implementation Details}\label{sec:implementation_details}

\subsection{Experimental Details}

	\textbf{Training} Throughout the experiments, we train our model with a learning rate of 0.0002, warmup of 5000 iterations, and gradient clipping by 1. For UNCSN++, we take $\sigma_{min}=10^{-3}$, and for NCSN++, we take $\sigma_{min}=10^{-2}$. On ImageNet32 training of the likelihood weighting and the variance weighting without Soft Truncation, we take $\epsilon=5\times 10^{-5}$, following the setting of \citet{song2021maximum}. Otherwise, we take $\epsilon=10^{-5}$. For other hyperparameters, we run our experiments according to \citet{song2020score,song2021maximum}.
	
	On datasets of resolution $32\times 32$, we use the batch size of 128, which consumes about 48Gb GPU memory. On STL-10 with resolution $48\times 48$, we use the batch size of 192, and on datasets of resolution $64\times 64$, we experiment with 128 batch size. The batch size for the datasets of resolution $256\times 256$ is 40, which takes nearly 120Gb of GPU memory. On the dataset of $1024\times1024$ resolution, we use the batch size of 16, which takes around 120Gb of GPU memory. We use five NVIDIA RTX-3090 GPU machines to train the model exceeding 48Gb, and we use a pair of NVIDIA RTX-3090 GPU machines to train the model that consumes less than 48Gb. 
	
	\textbf{Evaluation} We apply the EMA with rate of 0.999 on NCSN++/UNCSN++ and 0.9999 on DDPM++/UDDPM++. For the density estimation, we obtain the NLL performance by the Instantaneous Change of Variable \citep{song2020score, chen2018neural}. We choose $[\epsilon=10^{-5},T=1]$ to integrate the instantaneous change-of-variable of the probability flow as default, even for the ImageNet32 dataset. In spite that \citet{song2020score, song2021maximum} integrates the change-of-variable formula with the starting variable to be $\mathbf{x}_{0}$, Table 5 of \citet{kim2022maximum} analyzes that there are significant difference between starting from $\mathbf{x}_{\epsilon}$ and $\mathbf{x}_{0}$, if $\epsilon$ is not small enough. Therefore, we follow \citet{kim2022maximum} to compute $\mathbb{E}_{\mathbf{x}_{\epsilon}}\big[-\log{p_{\epsilon}^{\bm{\theta}}(\mathbf{x}_{\epsilon})}\big]$. However, to compare with the baseline models, we also evaluate the way \citet{song2020score, song2021maximum} and \citet{vahdat2021score} compute NLL. We denote the way of \citet{kim2022maximum} as \textit{after correction} and \citet{song2021maximum} as \textit{before correction}, throughout the appendix. We dequantize the data variable by the uniform dequantization \cite{ho2019flow++} for both \text{after}-and-\text{before corrections}. In the main paper, we only report the \textit{after correction} performances.
		
	For the sampling, we apply the Predictor-Corrector (PC) algorithm introduced in \citet{song2020score}. We set the signal-to-noise ratio as 0.16 on $32\times 32$ datasets, 0.17 on $48\times 48$ and $64\times 64$ datasets, 0.075 on 256$\times$256 sized datasets, and 0.15 on $1024\times 1024$. On datasets less than $256\times 256$ resolution, we iterate 1,000 steps for the PC sampler, while we apply 2,000 steps on the other high-dimensional datasets. Throughout the experiments for VESDE, we use the reverse diffusion \citep{song2020score} for the predictor algorithm and the annealed Langevin dynamics \citep{welling2011bayesian} for the corrector algorithm. For VPSDE, we use the Euler-Maruyama for the predictor algorithm, and we do not use any corrector algorithm.
	
	We compute the FID score \citep{song2020score} based on the modified Inception V1 network\footnote{\url{https://tfhub.dev/tensorflow/tfgan/eval/inception/1}} using the tensorflow-gan package for CIFAR-10 dataset, and we use the clean-FID \citep{parmar2021buggy} based on the Inception V3 network \citep{szegedy2016rethinking} for the remaining datasets. We note that FID computed by \cite{parmar2021buggy} reports a higher FID score compared to the original FID calculation\footnote{See \url{https://github.com/GaParmar/clean-fid} for the detailed experimental results.}.
	
\section{Additional Experimental Results}\label{sec:additional}
		
		\subsection{Ablation Study on Reconstruction Term}\label{sec:reconstruction_training}
	
	\begin{table}[t]
\centering
\caption{Ablation study of Soft Truncation with/without the reconstruction term when training on CIFAR-10 trained with DDPM++ (VP).}
\label{tab:ablation_reconstruction_appendix}
	\vskip -0.05in
\tiny
\begin{tabular}{lcc|cc|cc|c}
	\toprule
	\multirow{6}{*}[-2pt]{Loss} & \multirow{6}{*}[-2pt]{\shortstack{Soft\\Truncation}} & \multirow{6}{*}[-2pt]{\shortstack{Reconstruction\\Term for\\Training}} & \multicolumn{2}{c|}{NLL} & \multicolumn{2}{c|}{NELBO} & FID \\\cmidrule(lr){4-8}
	&&& \multirow{5}{*}{\shortstack{$\mathbb{E}_{\mathbf{x}_{0}}[-\log{p_{\epsilon}^{\bm{\theta}}(\mathbf{x}_{0})}]$\\(before correction)}} & \multirow{5}{*}{\shortstack{$\mathbb{E}_{\mathbf{x}_{\epsilon}}[-\log{p_{\epsilon}^{\bm{\theta}}(\mathbf{x}_{\epsilon})}]+R_{\epsilon}(\bm{\theta})$\\(after correction)}} & \multirow{5}{*}{\shortstack{$\mathcal{L}(\bm{\theta};g^{2},\epsilon)$\\(without\\residual)}} & \multirow{5}{*}{\shortstack{$\mathcal{L}(\bm{\theta};g^{2},\epsilon)$\\$+R_{\epsilon}(\bm{\theta})$\\(with\\residual)}} & \multicolumn{1}{c}{\multirow{5}{*}{ODE}} \\
	&&&&&&&\\
	&&&&&&&\\
	&&&&&&&\\
	&&&&&&&\\\midrule
	$\mathcal{L}(\bm{\theta};g^{2},\epsilon)$ & \xmark & \multicolumn{1}{c}{\xmark} & \multicolumn{1}{c}{2.97} & \multicolumn{1}{c}{3.03} & \multicolumn{1}{c}{3.11} & \multicolumn{1}{c}{3.13} & 6.70\\\cmidrule(lr){1-3}
	$\mathcal{L}(\bm{\theta};g^{2},\epsilon)+\mathbb{E}_{\mathbf{x}_{0},\mathbf{x}_{\epsilon}}\big[-\log{p(\mathbf{x}_{0}\vert\mathbf{x}_{\epsilon})}\big]$ & \xmark & \multicolumn{1}{c}{\cmark} & \multicolumn{1}{c}{3.01} & \multicolumn{1}{c}{2.99} & \multicolumn{1}{c}{3.07} & \multicolumn{1}{c}{3.09} & 6.93 \\\cmidrule(lr){1-3}
	$\mathcal{L}_{ST}(\bm{\theta};g^{2},\mathbb{P}_{1})=\mathbb{E}_{\mathbb{P}_{1}(\tau)}\big[\mathcal{L}(\bm{\theta};g^{2},\tau)\big]$ & \multirow{2}{*}{\cmark} & \multicolumn{1}{c}{\multirow{2}{*}{\xmark}} & \multicolumn{1}{c}{\multirow{2}{*}{2.98}} & \multicolumn{1}{c}{\multirow{2}{*}{3.01}} & \multicolumn{1}{c}{\multirow{2}{*}{3.08}} & \multicolumn{1}{c}{\multirow{2}{*}{3.08}} & \multirow{2}{*}{\textbf{3.96}} \\
	$=\mathbb{E}_{\mathbb{P}_{1}(\tau)}\big[\mathcal{L}(\bm{\theta};g^{2},\tau)\big]$ &&\multicolumn{1}{c}{} &&\multicolumn{1}{c}{}&&\multicolumn{1}{c}{}&\\\cmidrule(lr){1-3}
	$\mathbb{E}_{\mathbb{P}_{1}(\tau)}\big[\mathcal{L}(\bm{\theta};g^{2},\tau)+R_{\tau}(\bm{\theta})$ & \cmark & \multicolumn{1}{c}{\cmark} & \multicolumn{1}{c}{\textbf{2.95}} & \multicolumn{1}{c}{\textbf{2.98}} & \multicolumn{1}{c}{\textbf{3.04}} & \multicolumn{1}{c}{\textbf{3.04}} & 4.23 \\
	\bottomrule
\end{tabular}
\end{table}
	
		Table \ref{tab:ablation_reconstruction_appendix} presents that the training with the reconstruction term outperforms the training without the reconstruction term on NLL/NELBO with the sacrifice on sample generation. If $\tau$ is fixed as $\epsilon$, then the bound
		\begin{align*}
		\mathbb{E}_{\mathbf{x}_{0}}\big[-\log{p_{0}^{\bm{\theta}}(\mathbf{x}_{0})}\big]\le\mathcal{L}(\bm{\theta};g^{2},\tau)+\mathbb{E}_{\mathbf{x}_{0},\mathbf{x}_{\tau}}\big[-\log{p(\mathbf{x}_{0}\vert\mathbf{x}_{\tau})}\big]
		\end{align*}
		is tight enough to estimate the negative log-likelihood. However, if $\tau$ is a subject of random variable, then the bound is not tight to the negative log-likelihood, as evidenced in Figure \ref{fig:nelbo}-(b). On the other hand, if we do not count the reconstruction, then the bound becomes
		\begin{align*}
		\mathbb{E}_{\mathbf{x}_{0}}\big[-\log{p_{\tau}^{\bm{\theta}}(\mathbf{x}_{\tau})}\big]\le\mathcal{L}(\bm{\theta};g^{2},\tau),
		\end{align*}
		up to a constant, and this bound becomes tight regardless of $\tau$, which is evidenced in Figure \ref{fig:nelbo}-(c). This is why we call Soft Truncation as Maximum Perturbed Likelihood Estimation (MPLE).

\begin{table}[t]
\vskip 0.1in
\centering
	\caption{Ablation study of Soft Truncation for various weightings on CIFAR-10 and ImageNet32 with DDPM++ (VP).}
	\label{tab:ablation_weighting_function_appendix}
	\vskip -0.05in
	\tiny
	\begin{tabular}{llcccccc}
		\toprule
		\multirow{3}{*}[-2pt]{Dataset} & \multirow{3}{*}[-2pt]{Loss} & \multirow{3}{*}[-2pt]{\shortstack{Soft\\Truncation}} & \multicolumn{2}{c}{NLL} & \multicolumn{2}{c}{NELBO} & FID \\\cmidrule(lr){4-8}
		&&& \multirow{2}{*}{\shortstack{after\\correction}} & \multirow{2}{*}{\shortstack{before\\correction}} & \multirow{2}{*}{\shortstack{with\\residual}} & \multirow{2}{*}{\shortstack{without\\residual}} & \multirow{2}{*}{ODE} \\
		&&&&&&&\\\midrule
		\multirow{4}{*}[-2pt]{CIFAR-10} & $\mathcal{L}(\bm{\theta};g^{2},\epsilon)$ & \xmark & 3.03 & \textbf{2.97} & 3.13 & 3.11 & 6.70 \\
		& $\mathcal{L}(\bm{\theta};\sigma^{2},\epsilon)$ & \xmark & 3.21 & 3.16 & 3.34 & 3.32 & \textbf{3.90} \\
		& $\mathcal{L}(\bm{\theta};g_{\mathbb{P}_{1}}^{2},\epsilon)$ & \xmark & 3.06 & 3.02 & 3.18 & 3.14 & 6.11 \\
		& $\mathcal{L}_{ST}(\bm{\theta};g^{2},\mathbb{P}_{1})$ & \cmark & \textbf{3.01} & 2.98 & \textbf{3.08} & \textbf{3.08} & 3.96 \\\midrule
		\multirow{5}{*}[-2pt]{ImageNet32} & $\mathcal{L}(\bm{\theta};g^{2},\epsilon)$ & \xmark & 3.92 & 3.90 & 3.94 & 3.95 & 12.68 \\
		& $\mathcal{L}(\bm{\theta};\sigma^{2},\epsilon)$ & \xmark & 3.95 & 3.96 & 4.00 & 4.01 & 9.22 \\
		& $\mathcal{L}(\bm{\theta};g_{\mathbb{P}_{1}}^{2},\epsilon)$ & \xmark & 3.93 & 3.92 & 3.97 & 3.98 & 11.89 \\
		& $\mathcal{L}_{ST}(\bm{\theta};g^{2},\mathbb{P}_{1})$ & \cmark & \textbf{3.90} & \textbf{3.87} & 3.92 & 3.92 & 9.52 \\
		& $\mathcal{L}_{ST}(\bm{\theta};g^{2},\mathbb{P}_{0.9})$ & \cmark & \textbf{3.90} & 3.88 & \textbf{3.91} & \textbf{3.91} & \textbf{8.42} \\
		\bottomrule
	\end{tabular}
\end{table}

\begin{table}[t]
\vskip 0.1in
\centering
	\caption{Ablation study of Soft Truncation for various model architectures and diffusion SDEs on CelebA.}
	\label{tab:ablation_architecture_sde_appendix}
	\vskip -0.05in
	\tiny
	\begin{tabular}{lllcccccc}
		\toprule
		\multirow{3}{*}[-2pt]{SDE} & \multirow{3}{*}[-2pt]{Model} & \multirow{3}{*}[-2pt]{Loss} & \multicolumn{2}{c}{NLL} & \multicolumn{2}{c}{NELBO} & \multicolumn{2}{c}{FID} \\\cmidrule(lr){4-9}
		& && \multirow{2}{*}{\shortstack{after\\correction}} & \multirow{2}{*}{\shortstack{before\\correction}} & \multirow{2}{*}{\shortstack{with\\residual}} & \multirow{2}{*}{\shortstack{without\\residual}} & \multirow{2}{*}{PC} & \multirow{2}{*}{ODE} \\
		&&&&&&&&\\\midrule
		\multirow{2}{*}{VE} & \multirow{2}{*}{NCSN++} & $\mathcal{L}(\bm{\theta};\sigma^{2},\epsilon)$ & 3.41 & 2.37 & 3.42 & 3.96 & 3.95 & -\\
		& & $\mathcal{L}_{ST}(\bm{\theta};\sigma^{2},\mathbb{P}_{2})$ & 3.44 & 2.42 & 3.44 & 3.97 & 2.68 & -\\\midrule
		\multirow{2}{*}{RVE} & \multirow{2}{*}{UNCSN++} & $\mathcal{L}(\bm{\theta};g^{2},\epsilon)$ & 2.01 & 1.96 & \textbf{2.01} & 2.17 & 3.36 & -\\
		& & $\mathcal{L}_{ST}(\bm{\theta};g^{2},\mathbb{P}_{2})$ & \textbf{1.97} & \textbf{1.91} & 2.02 & 2.18 & \textbf{1.92} & -\\\midrule
		\multirow{8}{*}[-9pt]{VP} & \multirow{2}{*}{DDPM++} & $\mathcal{L}(\bm{\theta};\sigma^{2},\epsilon)$ & 2.14 & 2.07 & 2.21 & 2.22 & 3.03 & 2.32 \\
		& & $\mathcal{L}_{ST}(\bm{\theta};\sigma^{2},\mathbb{P}_{1})$ & 2.17 & 2.08 & 2.29 & 2.26 & 2.88 & \textbf{1.90}\\\cmidrule(lr){2-3}
		& \multirow{2}{*}{UDDPM++} & $\mathcal{L}(\bm{\theta};\sigma^{2},\epsilon)$ & 2.11 & 2.07 & 2.20 & 2.21 & 3.23 & 4.72\\
		& & $\mathcal{L}_{ST}(\bm{\theta};\sigma^{2},\mathbb{P}_{1})$ & 2.16 & 2.08 & 2.28 & 2.25 & 2.22 & 1.94\\\cmidrule(lr){2-3}
		& \multirow{2}{*}{DDPM++} & $\mathcal{L}(\bm{\theta};g^{2},\epsilon)$ & 2.00 & 1.93 & 2.09 & \textbf{2.09} & 5.31 & 3.95\\
		& & $\mathcal{L}_{ST}(\bm{\theta};g^{2},\mathbb{P}_{1})$ & 2.00 & 1.94 & 2.11 & 2.11 & 4.50 & 2.90\\\cmidrule(lr){2-3}
		& \multirow{2}{*}{UDDPM++} & $\mathcal{L}(\bm{\theta};g^{2},\epsilon)$ & 1.98 & 1.95 & 2.12 & 2.15 & 4.65 & 3.98\\
		& & $\mathcal{L}_{ST}(\bm{\theta};g^{2},\mathbb{P}_{1})$ & 2.00 & 1.94 & 2.10 & 2.10 & 4.45 & 2.97\\
		\bottomrule
	\end{tabular}
\end{table}

\begin{table}[t]
\centering
	\caption{Ablation study of Soft Truncation for various $\sigma_{min}$ (equivalently, $\epsilon$) on CIFAR-10 with UNCSN++ (RVE).}
	\label{tab:ablation_epsilon_appendix}
	\vskip -0.05in
	\tiny
	\begin{tabular}{lcccccc}
		\toprule
		\multirow{3}{*}[-2pt]{Loss} & \multirow{3}{*}[-2pt]{$\epsilon$} & \multicolumn{2}{c}{NLL} & \multicolumn{2}{c}{NELBO} & FID \\\cmidrule(lr){3-7}
		&& \multirow{2}{*}{\shortstack{after\\correction}} & \multirow{2}{*}{\shortstack{before\\correction}} & \multirow{2}{*}{\shortstack{with\\residual}} & \multirow{2}{*}{\shortstack{without\\residual}} & \multirow{2}{*}{ODE} \\
		&&&&&&\\\midrule
		\multirow{4}{*}{$\mathcal{L}(\bm{\theta};g^{2},\epsilon)$} & $10^{-2}$ & 4.64 & 4.02 & 4.69 & 5.20 & 38.82 \\
		& $10^{-3}$ & 3.51 & 3.20 & 3.52 & 3.90 & 6.21 \\
		& $10^{-4}$ & 3.05 & 2.98 & 3.08 & 3.24 & 6.33 \\
		& $10^{-5}$ & 3.03 & \textbf{2.97} & 3.13 & 3.11 & 6.70 \\\midrule
		\multirow{4}{*}{$\mathcal{L}_{ST}(\bm{\theta};g^{2},\mathbb{P}_{1})$} & $10^{-2}$ & 4.65 & 4.03 & 4.69 & 5.20 & 39.83 \\
		& $10^{-3}$ & 3.51 & 3.21 & 3.52 & 3.88 & 5.14 \\
		& $10^{-4}$ & 3.05 & 2.98 & 3.08 & 3.24 & 4.16 \\
		& $10^{-5}$ & \textbf{3.01} & 2.98 & \textbf{3.08} & \textbf{3.08} & \textbf{3.96} \\
		\bottomrule
	\end{tabular}
\end{table}

\begin{table}[t]
\vskip 0.1in
\centering
\caption{Ablation study of Soft Truncation for various $\mathbb{P}_{k}$ on CIFAR-10 trained with DDPM++ (VP).}
\label{tab:ablation_prior_appendix}
	\vskip -0.05in
\tiny
\begin{tabular}{lccccc}
	\toprule
	\multirow{3}{*}[-2pt]{Loss} & \multicolumn{2}{c}{NLL} & \multicolumn{2}{c}{NELBO} & FID \\\cmidrule(lr){2-6}
	& \multirow{2}{*}{\shortstack{after\\correction}} & \multirow{2}{*}{\shortstack{before\\correction}} & \multirow{2}{*}{\shortstack{with\\residual}} & \multirow{2}{*}{\shortstack{without\\residual}} & \multirow{2}{*}{ODE} \\
	&&&&&\\\midrule
	$\mathcal{L}_{ST}(\bm{\theta};g^{2},\mathbb{P}_{0})$ & 3.24 & 3.16 & 3.39 & 3.34 & 6.27 \\
	$\mathcal{L}_{ST}(\bm{\theta};g^{2},\mathbb{P}_{0.8})$ & 3.03 & 3.00 & \textbf{3.05} & \textbf{3.05} & 3.61 \\
	$\mathcal{L}_{ST}(\bm{\theta};g^{2},\mathbb{P}_{0.9})$ & 3.03 & 2.99 & 3.13 & 3.13 & \textbf{3.45} \\
	$\mathcal{L}_{ST}(\bm{\theta};g^{2},\mathbb{P}_{1})$ & \textbf{3.01} & 2.98 & 3.08 & 3.08 & 3.96 \\
	$\mathcal{L}_{ST}(\bm{\theta};g^{2},\mathbb{P}_{1.1})$ & 3.02 & 2.99 & 3.09 & 3.10 & 3.98 \\
	$\mathcal{L}_{ST}(\bm{\theta};g^{2},\mathbb{P}_{1.2})$ & 3.03 & 2.99 & 3.09 & 3.09 & 3.98 \\
	$\mathcal{L}_{ST}(\bm{\theta};g^{2},\mathbb{P}_{2})$ & \textbf{3.01} & 2.97 & 3.10 & 3.09 & 6.31 \\
	$\mathcal{L}_{ST}(\bm{\theta};g^{2},\mathbb{P}_{3})$ & 3.02 & 2.96 & 3.09 & 3.09 & 6.54 \\\midrule
	$\mathcal{L}_{ST}(\bm{\theta};g^{2},\mathbb{P}_{\infty})$ & \multirow{2}{*}{\textbf{3.01}} & \multirow{2}{*}{\textbf{2.95}} & \multirow{2}{*}{3.09} & \multirow{2}{*}{3.07} & \multirow{2}{*}{6.70} \\
	$=\mathcal{L}(\bm{\theta};g^{2},\epsilon)$ &&&&&\\
	\bottomrule
\end{tabular}
\end{table}

\begin{table}[t]
\vskip 0.1in
\centering
\caption{Ablation study of Soft Truncation for CIFAR-10 trained with DDPM++ when a diffusion is combined with a normalizing flow \cite{kim2022maximum}. We use $\mathbb{P}([a,b])=\frac{1}{2}1_{[a,b]}(\epsilon)+\frac{1}{2}\mathbb{P}_{0.9}([a,b])$.}
\label{tab:ablation_indm_appendix}
	\vskip -0.05in
\tiny
\begin{tabular}{lcccccc}
	\toprule
	\multirow{3}{*}[-2pt]{Loss} & \multicolumn{2}{c}{NLL} & \multicolumn{2}{c}{NELBO} & FID \\\cmidrule(lr){2-6}
	& \multirow{2}{*}{\shortstack{after\\correction}} & \multirow{2}{*}{\shortstack{before\\correction}} & \multirow{2}{*}{\shortstack{with\\residual}} & \multirow{2}{*}{\shortstack{without\\residual}} & \multirow{2}{*}{ODE} \\
	&&&&&\\\midrule
	$\mathcal{L}(\bm{\theta};g^{2},\epsilon)$ & 2.97 & 2.94 & 2.97 & 2.96 & 6.06 \\
	$\mathcal{L}(\bm{\theta};\sigma^{2},\epsilon)$ & 3.17 & 3.11 & 3.23 & 3.18 & 3.61 \\
	$\mathcal{L}(\bm{\theta};g^{2},\mathbb{P})$ & 3.01 & 2.98 & 3.02 & 3.01 & 3.89 \\
	\bottomrule
\end{tabular}
\end{table}
		
		\subsection{Full Tables}\label{sec:full_tables}
		
		Tables \ref{tab:ablation_weighting_function_appendix}, \ref{tab:ablation_architecture_sde_appendix}, \ref{tab:ablation_epsilon_appendix}, \ref{tab:ablation_prior_appendix}, and \ref{tab:ablation_indm_appendix} present the full list of performances for Soft Truncation.
		
		\subsection{Generated Samples}
		
		Figure \ref{fig:denoising} shows how images are created from the trained model, and Figures from \ref{fig:cifar10} to \ref{fig:celebahq256} present non-cherry picked generated samples of the trained model.
		
	\begin{figure}[t]
		\centering
		\includegraphics[width=\linewidth]{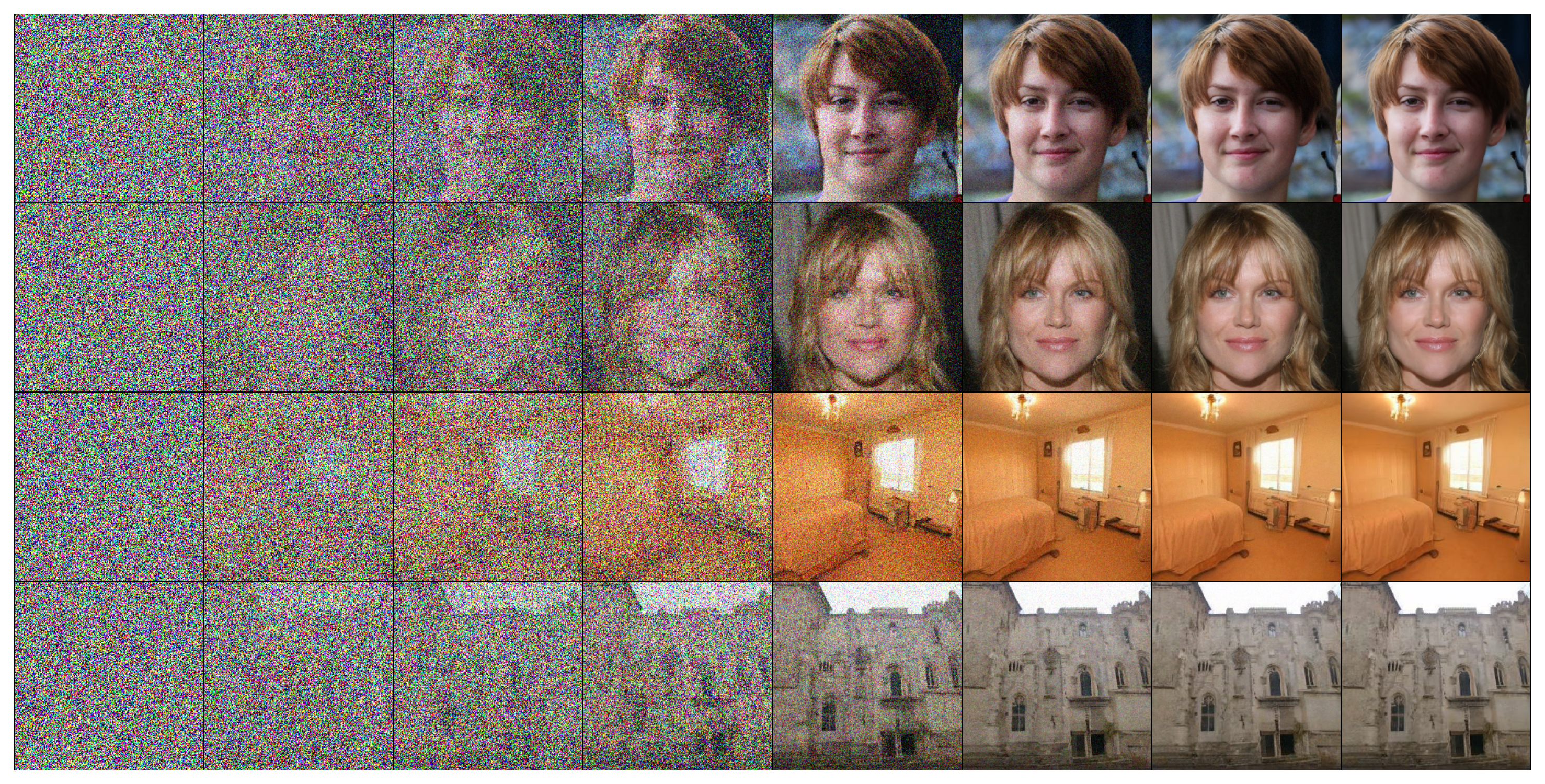}
		\caption{Image generation by denoising via predictor-corrector sampler.}
		\label{fig:denoising}
	\end{figure}
	
	\begin{figure}[t]
		\centering
		\includegraphics[width=\linewidth]{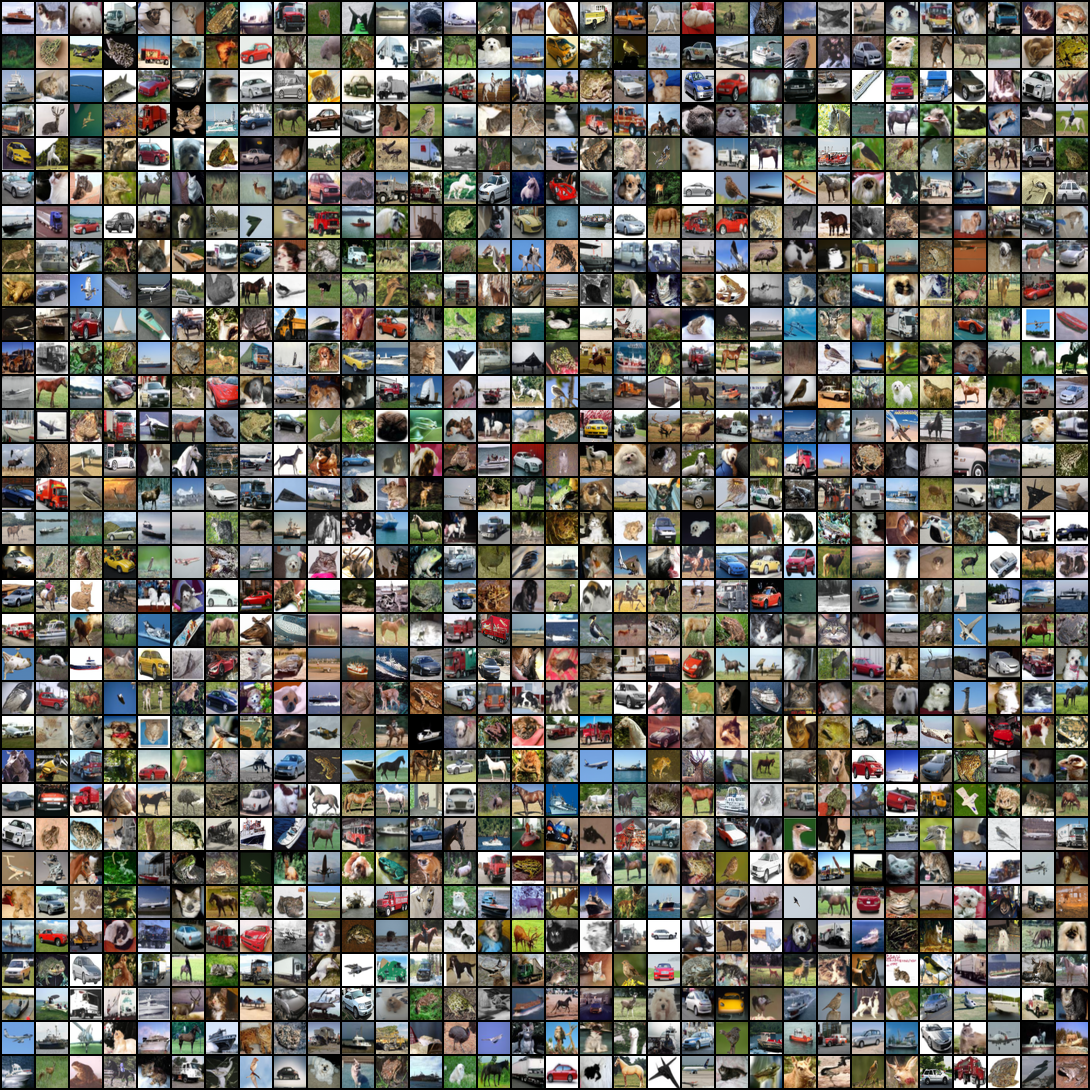}
		\caption{Random samples of CIFAR-10.}
		\label{fig:cifar10}
	\end{figure}
	
	\begin{figure}[t]
		\centering
		\includegraphics[width=\linewidth]{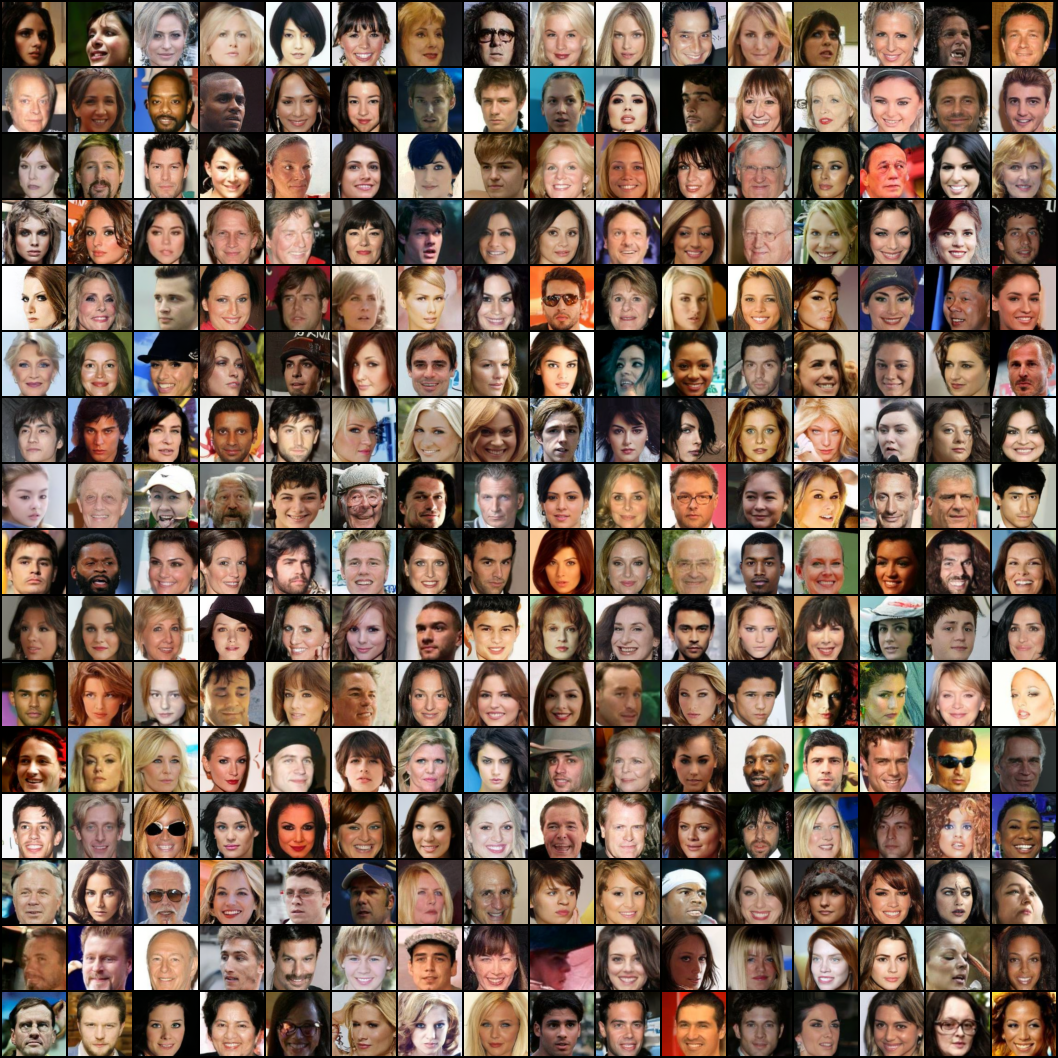}
		\caption{Random samples on CelebA.}
		\label{fig:celeba}
	\end{figure}
	
	\begin{figure}[t]
		\centering
		\includegraphics[width=\linewidth]{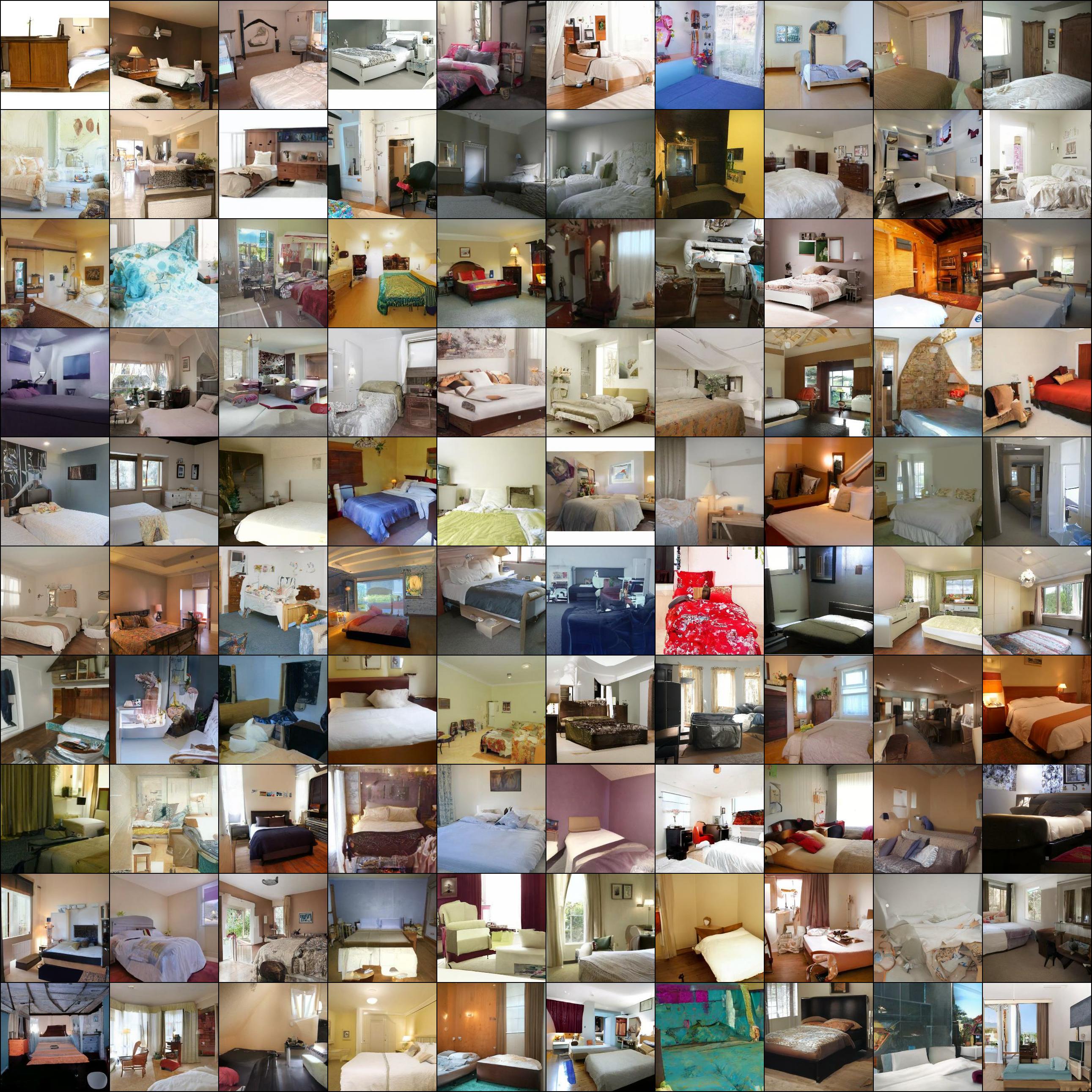}
		\caption{Random samples on LSUN Bedroom.}
		\label{fig:bedroom}
	\end{figure}
	
	\begin{figure}[t]
		\centering
		\includegraphics[width=\linewidth]{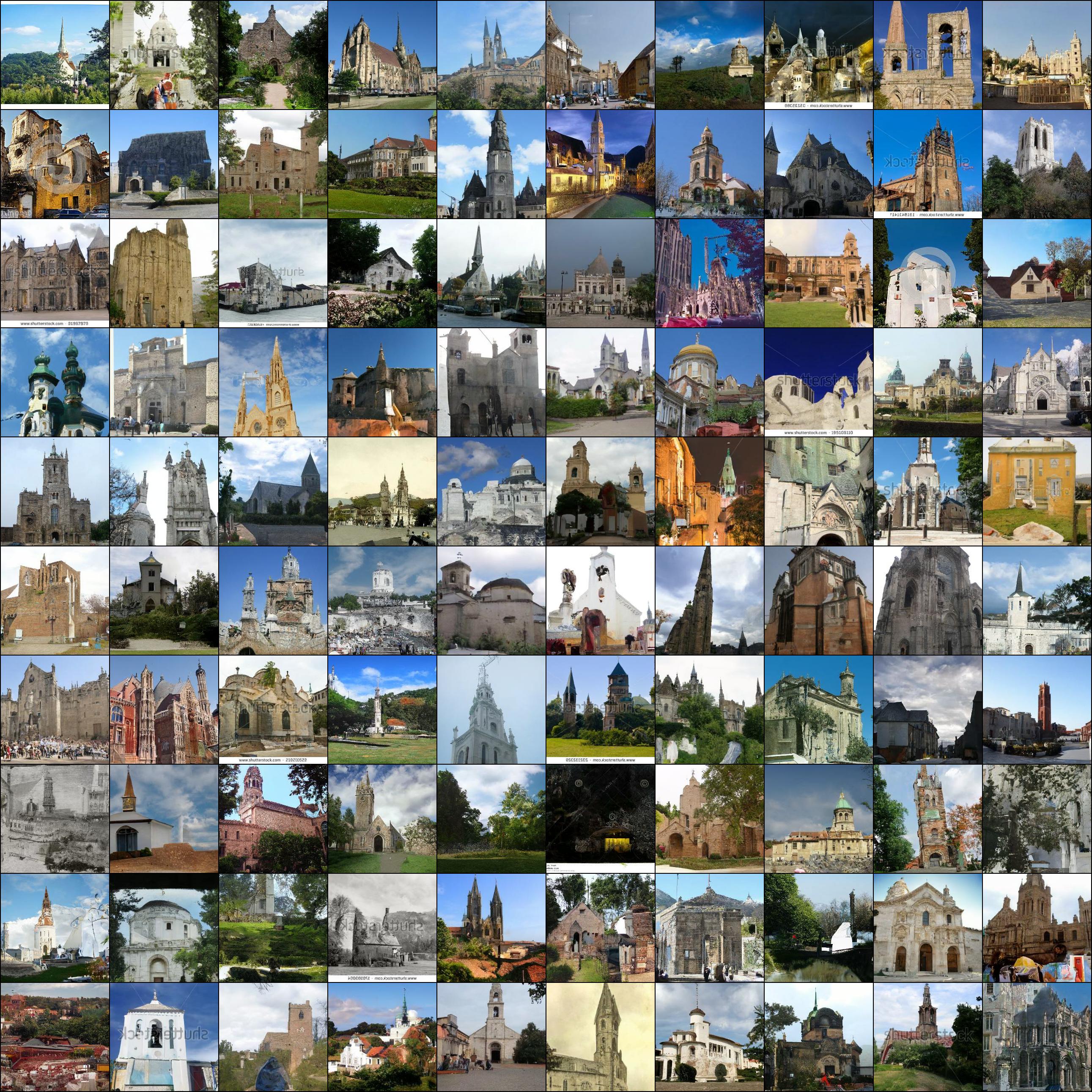}
		\caption{Random samples on LSUN Church.}
		\label{fig:church}
	\end{figure}
	
	\begin{figure}[t]
		\centering
		\includegraphics[width=\linewidth]{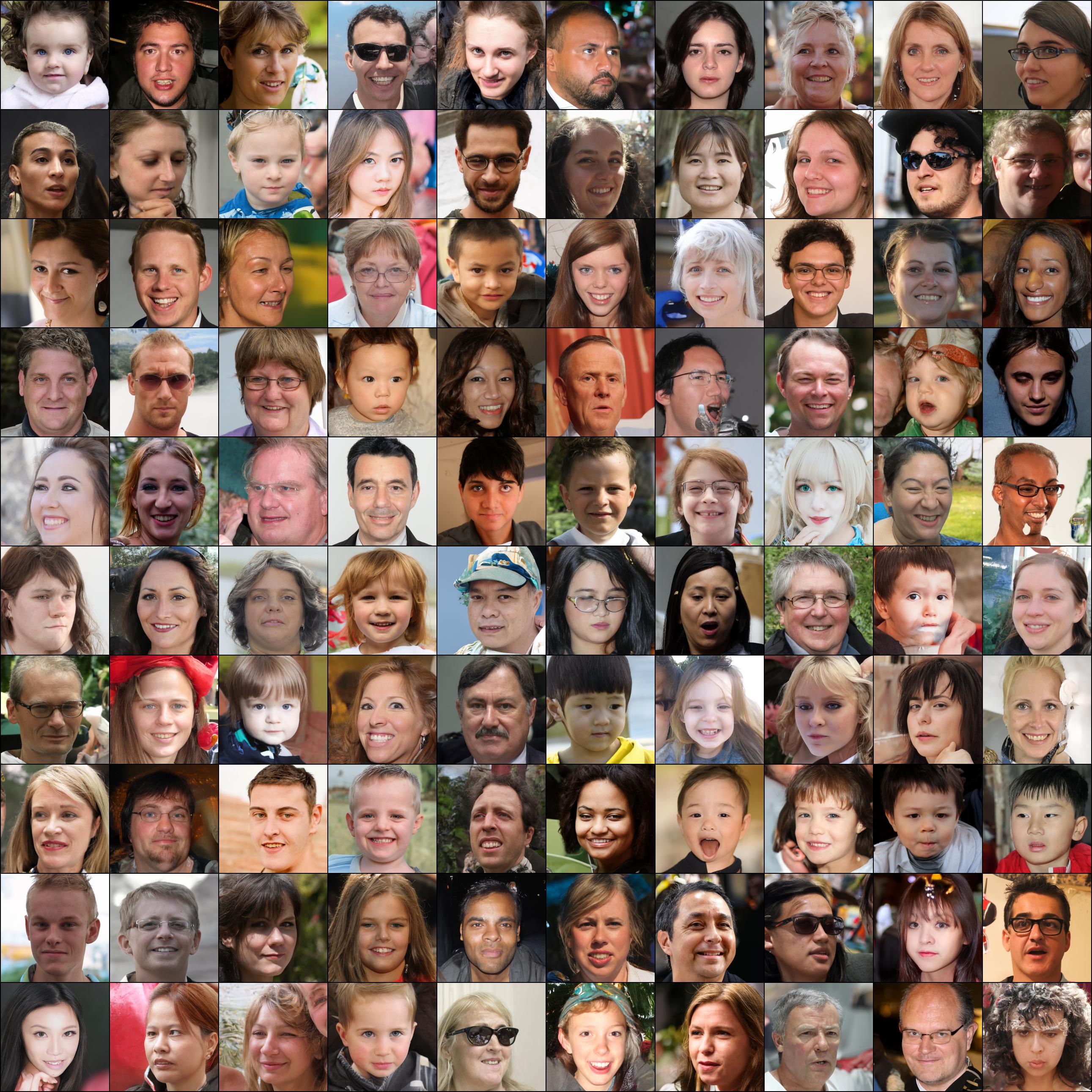}
		\caption{Random samples on FFHQ 256.}
		\label{fig:ffhq256}
	\end{figure}
	
	\begin{figure}[t]
		\centering
		\includegraphics[width=\linewidth]{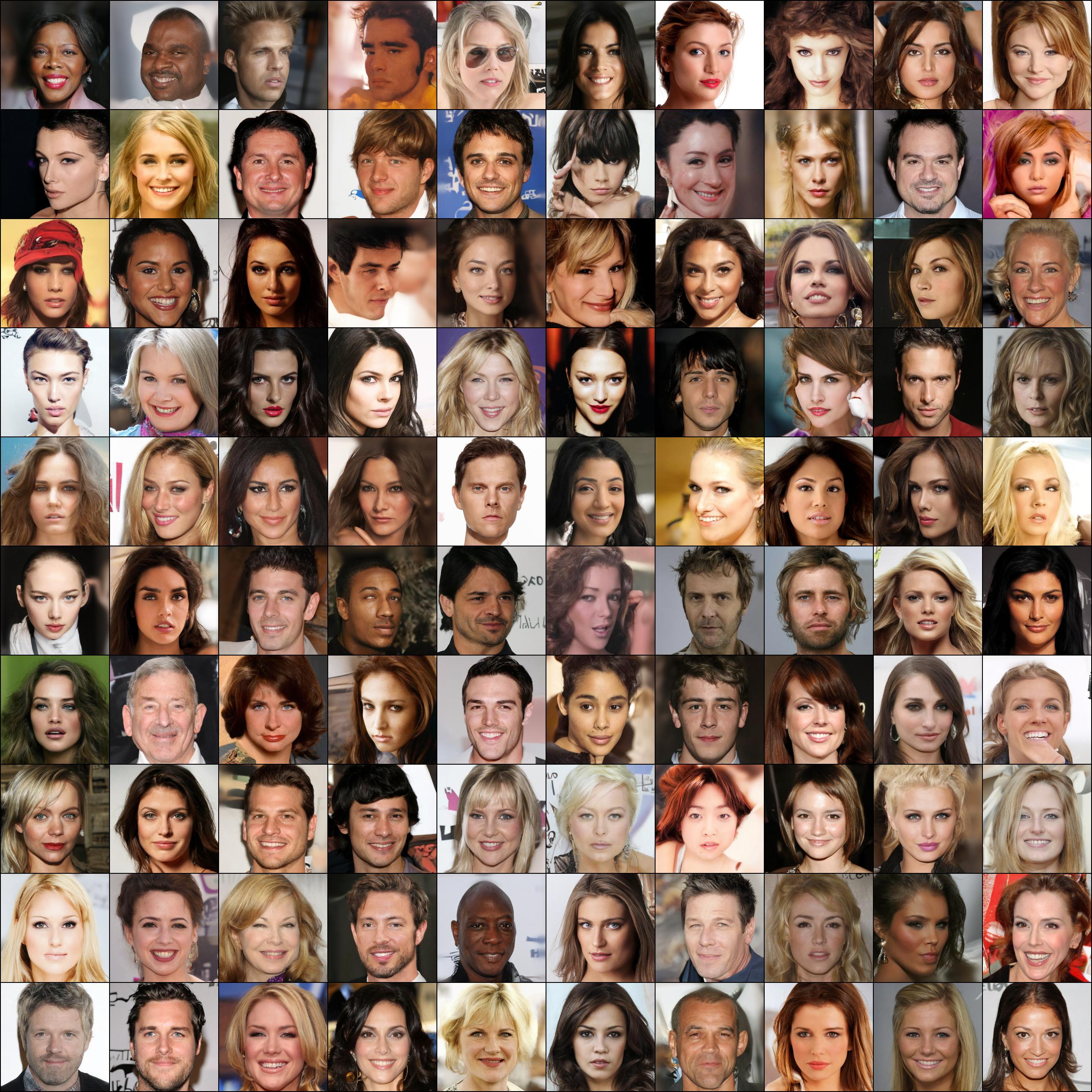}
		\caption{Random samples on CelebA-HQ 256.}
		\label{fig:celebahq256}
	\end{figure}


\end{document}